\documentclass[nohyperref]{article}

\usepackage{microtype}
\usepackage{graphicx}
\usepackage{subfigure}
\usepackage{booktabs} % for professional tables

\usepackage{hyperref}

\usepackage[accepted]{icml2022}

\usepackage{amsmath}
\usepackage{amssymb}
\usepackage{mathtools}
\usepackage{amsthm}

\usepackage[capitalize,noabbrev]{cleveref}

\usepackage{thmtools, thm-restate}
\usepackage{mathrsfs} %pour mathscr

\usepackage{xcolor}
\usepackage{soul}

\definecolor{dgreen}{rgb}{0.00,0.49,0.00}
\definecolor{dblue}{rgb}{0,0.08,0.75}
\hypersetup{
    colorlinks = true,
    citecolor=dgreen,
    urlcolor=dblue,
    linkcolor=dblue
}

\renewcommand{\P}{\mathbb{P}}
\newcommand{\R}{\mathbb{R}}
\newcommand{\cE}{\mathcal{E}}
\newcommand{\cD}{\mathcal{D}}
\newcommand{\cH}{\mathcal{H}}
\newcommand{\cM}{\mathcal{M}}

\newcommand{\Q}{\mathbb{Q}}
\newcommand{\cX}{\mathcal{X}}
\newcommand{\N}{\mathbb{N}}
\newcommand{\cF}{\mathcal{F}}
\newcommand{\E}{\mathbb{E}}
\newcommand{\la}{\lambda}
\newcommand{\cA}{\mathcal{A}}
\newcommand{\cB}{\mathcal{B}}
\newcommand{\cU}{\mathcal{U}}
\newcommand{\be}{\beta}
\newcommand{\al}{\alpha}

\DeclareMathOperator*{\argmin}{argmin}

\declaretheorem[name=Theorem,refname=Theorem]{theorem}
\declaretheorem[name=Lemma,sibling=theorem]{lemma}

\declaretheorem[name=Proposition,refname=Proposition,sibling=theorem]{proposition}
\declaretheorem[name=Remark,sibling=theorem]{remark}
\declaretheorem[name=Corollary,refname=Corollary,sibling=theorem]{corollary}
\declaretheorem[name=Definition,refname=Definition]{definition}

\crefname{assumption}{Assumption}{Assumptions}
\crefname{equation}{}{}
\Crefname{equation}{Eq.}{Eqs.}
\crefname{figure}{Fig.}{Fig.}
\crefname{table}{Table}{Tables}
\crefname{section}{Sec.}{Sec.}
\crefname{theorem}{Thm.}{Thm.}
\crefname{proposition}{Prop.}{Prop.}
\crefname{fact}{Fact}{Facts}
\crefname{lemma}{Lemma}{Lemmas}
\crefname{corollary}{Cor.}{Cor.}
\crefname{example}{Example}{Examples}
\crefname{remark}{Remark}{Remarks}
\crefname{algorithm}{Alg.}{Algorightms}

\crefname{appendix}{Appendix}{Appendices}

\usepackage[textsize=tiny]{todonotes}

\icmltitlerunning{Distribution Regression with Sliced Wasserstein Kernels}

\begin{document}

\twocolumn[
\icmltitle{Distribution Regression with Sliced Wasserstein Kernels}

% It is OKAY to include author information, even for blind
% submissions: the style file will automatically remove it for you
% unless you've provided the [accepted] option to the icml2022
% package.

% List of affiliations: The first argument should be a (short)
% identifier you will use later to specify author affiliations
% Academic affiliations should list Department, University, City, Region, Country
% Industry affiliations should list Company, City, Region, Country

% You can specify symbols, otherwise they are numbered in order.
% Ideally, you should not use this facility. Affiliations will be numbered
% in order of appearance and this is the preferred way.
%\icmlsetsymbol{equal}{*}

\begin{icmlauthorlist}
\icmlauthor{Dimitri Meunier}{gatsby}
\icmlauthor{Massimiliano Pontil}{iit,ucl}
\icmlauthor{Carlo Ciliberto}{ucl}
%\icmlauthor{Firstname1 Lastname1}{equal,yyy}
%\icmlauthor{Firstname2 Lastname2}{equal,yyy,comp}
%\icmlauthor{Firstname3 Lastname3}{comp}
%\icmlauthor{Firstname4 Lastname4}{sch}
%\icmlauthor{Firstname5 Lastname5}{yyy}
%\icmlauthor{Firstname6 Lastname6}{sch,yyy,comp}
%\icmlauthor{Firstname7 Lastname7}{comp}
%\icmlauthor{}{sch}
%\icmlauthor{Firstname8 Lastname8}{sch}
%\icmlauthor{Firstname8 Lastname8}{yyy,comp}
%\icmlauthor{}{sch}
%\icmlauthor{}{sch}
\end{icmlauthorlist}

%\icmlaffiliation{yyy}{Department of XXX, University of YYY, Location, Country}
%\icmlaffiliation{comp}{Company Name, Location, Country}
%\icmlaffiliation{sch}{School of ZZZ, Institute of WWW, Location, Country}
\icmlaffiliation{gatsby}{Gatsby Computational Neuroscience Unit, University College London, London, United Kingdom}
\icmlaffiliation{ucl}{Department of Computer Science, University College London, London, United Kingdom}
\icmlaffiliation{iit}{Italian Institute of Technology, Genoa, Italy}

%\icmlcorrespondingauthor{Firstname1 Lastname1}{first1.last1@xxx.edu}
%\icmlcorrespondingauthor{Firstname2 Lastname2}{first2.last2@www.uk}
\icmlcorrespondingauthor{Dimitri Meunier}{dimitri.meunier.21@ucl.ac.uk}

% You may provide any keywords that you
% find helpful for describing your paper; these are used to populate
% the "keywords" metadata in the PDF but will not be shown in the document
\icmlkeywords{Machine Learning, ICML, Kernel Methods}

\vskip 0.3in
]

% this must go after the closing bracket ] following \twocolumn[ ...

% This command actually creates the footnote in the first column
% listing the affiliations and the copyright notice.
% The command takes one argument, which is text to display at the start of the footnote.
% The \icmlEqualContribution command is standard text for equal contribution.
% Remove it (just {}) if you do not need this facility.

\printAffiliationsAndNotice{}  % leave blank if no need to mention equal contribution
%\printAffiliationsAndNotice{\icmlEqualContribution} % otherwise use the standard text.

\begin{abstract}
The problem of learning functions over spaces of probabilities -- or distribution regression -- is gaining significant interest in the machine learning community. A key challenge behind this problem is to identify a suitable representation capturing all relevant properties of the underlying functional mapping. A principled approach to distribution regression is provided by kernel mean embeddings, which lifts kernel-induced similarity on the input domain at the probability level. This strategy effectively tackles the two-stage sampling nature of the problem, enabling one to derive estimators with strong statistical guarantees, such as universal consistency and excess risk bounds. However, kernel mean embeddings implicitly hinge on the maximum mean discrepancy (MMD), a metric on probabilities, which may fail to capture key geometrical relations between distributions. In contrast, optimal transport (OT) metrics, are potentially more appealing. In this work, we propose an OT-based estimator for distribution regression. We build on the Sliced Wasserstein distance to obtain an OT-based representation. We study the theoretical properties of a kernel ridge regression estimator based on such representation, for which we prove universal consistency and excess risk bounds. Preliminary experiments complement our theoretical findings by showing the effectiveness of the proposed approach and compare it with MMD-based estimators.
\end{abstract}

\vfill\null

\section{Introduction}
\label{intro}
Distribution regression studies the problem of learning a real-valued function defined on the space of probability distributions from a sequence of input/output observations.
A peculiar feature of this setting is that, typically, the input distributions are not observed directly but only accessed through finite samples, adding complexity to the problem \citep{poczos2013distribution}. A typical example in the literature \citep[see][]{szabo2016learning, fang2020optimal} is the task of predicting some health indicator of patients, labeled by $t$, based on their distribution $\P_t$ over results of clinical investigations, e.g. blood measurements. The underlying distribution of test results for each patient is observed only through repeated medical tests $x_{t,i} \sim \P_t$, $i=1,\ldots,n_t$, forming an approximation of $\P_t$ through $\hat\P_{t,n_t} = \frac{1}{n_t}\sum_{i=1}^{n_t} \delta_{x_i}$. Given labeled examples $(\hat\P_{t,n_t}, y_t)_{t=1}^T$, where $y_t$ is the health indicator, we seek to learn an accurate mapping $f: \hat\P_{n} \to y$ for a new patient identified by $\P$. Other examples of interest for the machine learning community include the task of predicting summary statistics of distributions such as the entropy or the number of modes \citep{oliva2014fast} or for conditional meta-learning \cite{denevi2020advantage}.

Historically, distribution regression has been tackled by means of a two-stage procedure, that build upon the kernel mean embedding machinery \citep{smola2007hilbert, muandet2017kernel}. A first positive definite (p.d.) kernel \citep{smola1998learning} defined on the underlying data space is used to embed the distribution $\hat\P_{t,n_t}$ into an Hilbert space, and then a second kernel is used as a similarity measure between the embedded distributions to perform Kernel Ridge Regression (KRR). This approach has been theoretically investigated in \citet{szabo2016learning} and further analyzed in \citet{fang2020optimal}. More recently, a SGD variant \cite{mucke2021stochastic} and a robust variant \cite{yu2021robust} have been suggested. 

The two stages above allow to build a p.d. kernel on the space of probability distributions and provide an example of distance substitution kernel, a class of kernels that can be built from distances on metric spaces \citep{haasdonk2004learning}. In that particular setting the chosen distance is the Maximum Mean Discrepancy (MMD), a metric on probability distributions \cite{gretton2012kernel}. A main theme of this paper is that distance substitution kernels based on optimal transport (with the Wasserstein distance) might be preferable over MMD-based ones, as they may better capture the geometry of the space of distributions \cite{peyre2019computational}.

The above observations lead us to introduce a Wasserstein-based p.d. kernel for distribution regression. As the original Wasserstein distance cannot be used to build a p.d. kernel on the space of distributions our construction exploits the Sliced-Wasserstein (SW) distance \cite{bonneel2015sliced}, a computationally cheapest alternative to the standard Wasserstein distance. SW is raising significant interest in the machine learning community due to its attractive topological and statistical properties \citep{nadjahi2019asymptotic, nadjahi2020statistical, nadjahi2021sliced} and
its strong numerical performances \cite{kolouri2018sliced, kolouri2019generalized, deshpande2018generative, deshpande2019max, lee2019sliced} and references therein.

A SW-based kernel was first introduced in \citet{kolouri2016sliced}, but their construction was
restricted to the subset of absolutely continuous distributions. Hence, this approach is not 
%fitted 
suitable for the distribution regression problem where we need to handle empirical distributions of the form $\hat\P_{t,n}$. In parallel, a SW-based kernel has been developed for persistence diagrams in topological data analysis \citet{carriere2017sliced}. \citet{bachoc2017gaussian, bui2018distribution} defined a p.d. kernel based on the Wasserstein distance valid for any probability measure defined on $\R$; while \citet{bachoc2020gaussian} used a similar extension to ours to $\R^r$, $r \geq 1$ focusing on Gaussian process distribution regression. In this paper however, we focus on statistical properties of the estimator within a kernel method setting.

Finally, there are deep learning heuristic methods to tackle the problem of learning from distributions such as deep sets \citet{zaheer2017deep} and references therein. While they might perform well in practice, they are harder to analyse theoretically and fall outside the scope of kernel methods which is the focus of this present work.

\noindent {\bf Contributions.}
~
This paper provides four main contributions: 1) we present a Wasserstein-based p.d. kernel valid for any probability distribution on $\R^r$; 2) we show that such kernel is universal in the sense of \citet{christmann2010universal}, allowing us to asymptotically learn any distribution regression function; 3) we provide excess risk bounds for the KRR estimator based on such kernels, which implies its consistency; 4) finally, we provide empirical evidences that the SW kernels behave better than MMD kernels on a number of distribution regression tasks. 

\noindent {\bf Paper organisation.}
~
The rest of the paper is organized as follows: \cref{sec:dist_regr} introduces the distribution regression setting and \cref{sec:dist_subst_kernels} reviews distance substitution kernels as a way to approach such problems. \cref{section:SW} shows that the SW distance can be used to obtain universal distance substitution kernels. \cref{sec:excess_risk} characterizes the generalization properties of distribution regression estimators based on distance substitution kernels. \cref{sec:exp} reports empirical evidences supporting the adoption of SW kernels in practical applications. 

\section{Distribution regression} \label{sec:dist_regr}
\noindent {\bf Tools and notations.}~
%\subsection{Tools and notations}
By convention, vectors $v\in\R^r$ are seen as $r\times 1$ matrices and $\|v\|$ denotes their Euclidean norm. $A^{\top}$ denotes the transpose of any $r \times k$ matrix $A$, $I_r \in \R^{r \times r}$ the identity matrix and $I$ the identity operator. For a measurable space $(\Omega,\cA)$ we denote by $\cM_+^1(\Omega)$ the space of probability measures on $(\Omega,\cA)$. Throughout most of the paper we use $\cM_+^1(\Omega_r)$ with $\Omega_r$ a compact subset of $\R^r$. For two measurable spaces $(\Omega,\cA), (\cX,\cB)$ and a measurable map $T: \Omega \to \cX$, the push-forward operator $T_{\#}: \cM_+^1(\Omega) \to \cM_+^1(\cX)$ is defined such that for all $\alpha \in \cM_+^1(\Omega)$ and for all measurable set $B \in \cB$, $T_{\#}(\alpha)(B) = \alpha(T^{-1}(B))$. $S^{r-1} = \{x \in \R^r \mid \|x\|_2 = 1\}$ denotes the unit sphere in $\R^r$. $\lambda^r$ is the Lebesgue measure on $\R^r$ and $\sigma^r$ is the uniform probability measure on the unit sphere, we use $dx$ instead of $\lambda^r(dx)$ and $d\theta$ instead of $\sigma^r(d\theta)$ when there is no ambiguity. For a vector $\theta \in S^{r-1}$, $\theta^*: \R^r \to \R, x \mapsto \langle x, \theta \rangle$. For a univariate probability measure $\al \in \cM_+^1(\R)$, we denote by $F_{\alpha}$ its cumulative distribution function, $F_{\alpha}(x) = \alpha((-\infty,x]) \in [0,1]$, by $F_{\alpha}^{[-1]}$ its generalised inverse cumulative distribution function, $F_{\alpha}^{[-1]}(x) = \inf\{t \in \R : F_{\alpha}(t) \geq x\}$,  and by $F_{\alpha}^{-1}$ its inverse cumulative distribution function, if it exists. On a measure space $(\Omega, \cA, \mu)$ we denote by $L^{2}(\Omega,\mu)$, abbreviated $L^{2}_{\mu}$ if the context is clear, the Hilbert space of square-integrable functions with respect to $\mu$ from $\Omega$ to $\R$ with norm $\|\cdot\|_{L^{2}_{\mu}}$ induced by the inner product $\langle f, g\rangle_{L^{2}_{\mu}}=\int_{\Omega} f(w) g(w) d\mu(w)$. Finally, for a compact set $\Omega$, $\cU(\Omega)$ denotes the uniform distribution on $\Omega$.

\noindent {\bf Problem set-up.}~
Linear regression considers Euclidean vectors as covariate while kernel regression extends the framework to any covariate space $X$ on which we can define a p.d. kernel $K: X \times X \to \R$ \citep{hofmann2008kernel}. In distribution regression, $X$ is a space of probability measures with the additional difficulty that the covariates $\P \in X$ cannot be directly observed \citep{poczos2013distribution}. Each $\P \in X$ is only accessed through i.i.d. samples $\{x_i\}_{i=1}^n \sim \P$, which forms a noisy measurement of $\P$. The formal problem set up is as follows. $X:=\cM_+^1(\Omega_r)$ is a space of probability distributions, we refer to $X$ as the input space and $\Omega_r$ as the (compact) data space. $Y \subseteq \R$ is a real subset and $\rho$ is an unknown distribution on $Z:= X \times Y$. $\rho_X$ is its marginal distribution on $X$ and $\rho_{Y \mid X}$ is its conditional distribution: $\rho(\P,y) = \rho_{Y \mid X}(y \mid \P)\rho_{X}(\P)$. The goal is to find a measurable map $f:X \to Y$ with low expected risk
%with respect to $\rho$:
\begin{equation}
    \cE(f) = \int_Z (f(\P) - y)^2 d\rho(\P,y).
\end{equation}
The minimizer of this functional (see \cref{prop:a_bayes_f}) is the regression function
\begin{equation}
    f_{\rho}(\P) = \int_Y y d\rho_{Y \mid X}(y \mid \P), \qquad \forall ~\P \in X,
\end{equation}
which is not accessible since $\rho$ is unknown. To approximate $f_{\rho}$ we draw a set of input-label pairs $\cD = (\P_t,y_t)_{t=1}^T \sim^{i.i.d} \rho$. We call this the \textbf{first stage sampling}. If we were in a standard regression setting, i.e. the inputs $\P_t$ were directly observable, given a p.d. kernel $K$ on $X$, one approach would be to apply KRR by minimizing the regularized empirical risk:
\begin{equation} \label{eq:risk1}
    \cE_{T,\lambda}(f) = \frac{1}{T}\sum_{t=1}^T \left(f\left(\P_t\right) - y_t\right)^2 + \lambda \|f\|_{\cH_K}^2
\end{equation}
where $\cH_K$ is te Reproducing Kernel Hilbert Space (RKHS) induced by $K$. We refer to this quantity as the \emph{first stage empirical risk} and denote its minimizer $f_{\cD,\lambda}$. The difference with standard regression comes from the \textbf{second stage of sampling} were instead of directly observing the covariates we receive for each input a bag of samples: $\hat{\cD} = \{((x_{t,i})_{i=1}^n,y_t)\}_{t=1}^T$, $(x_{t,i})_{i=1}^n \sim^{i.i.d} \P_t$. For ease of notations, we assume without loss of generality that each bag has the same number of points $n$. The strategy is then to plug the empirical distributions $\hat{\P}_{t,n} = \frac{1}{n}\sum_{i=1}^n \delta_{x_{t,i}}$ in the first stage risk leading to the \emph{second stage risk}: 
\begin{equation} \label{eq:risk2}
    \cE_{T,n,\lambda}(f) = \frac{1}{T}\sum_{t=1}^T \left(f\left(\hat{\P}_{t,n}\right) - y_t\right)^2 + \lambda \|f\|_{\cH_K}^2.
\end{equation}
We denote by $f_{\hat{\cD},\lambda}$ its minimizer.

\noindent {\bf Kernel distribution regression.}~
Equipped with a p.d. kernel $K: X \times X \to \R$, minimizing $\cE_{T,n,\lambda}$ over $\cH_{K}$ leads to a practical algorithm. For a new test point $\P \in X$,
\begin{equation} \label{eq:sol_KRR}
    f_{\hat{\cD},\lambda}(\P) = (y_1, \ldots, y_T)(K_T + \lambda T I_T)^{-1}k_{\P}
\end{equation}
where $[K_T]_{t,l} = K(\hat{\P}_{t,n},\hat{\P}_{l,n}) \in \R^{T \times T}$ and \\ $k_{\P} = (K(\P, \hat{\P}_{t,n}), \ldots, K(\P, \hat{\P}_{T,n}))^{\top} \in \R^T$ (see \cref{prop:a_KRR}). Let us note that in practice the new test points will also have the form of empirical distributions $\hat\P_n = \frac{1}{n}\sum_{i=1}^n \delta_{x_{i}}$. The algorithm will differ by which kernel is used on $X$.

\section{Distance substitution kernels} \label{sec:dist_subst_kernels}
In this section, we derive two types of kernels that can be defined on $X$. We approach this through the prism of distance substitution kernels. Let us assume that we have a set $M$ on which we want to define a p.d. kernel and we are given a pseudo-distance $d$ on $M$. Pseudo means that $d$ might not satisfy $d(x,y)=0 \implies x=y$ for all $(x,y) \in M^2$. One could think of using $d$ to define a p.d kernel on $M$ by substituting $d$ to the Euclidean distance for usual Euclidean kernels \citep{haasdonk2004learning}. For example,
\begin{equation} \label{eq:gauss_like}
    K_{\rm Gauss}(x,y) = e^{-\gamma d(x,y)^2},
\end{equation}
(with $x,y\in M$) generalizes the well-known Gaussian kernel $K_{\rm Gauss}(x,y) = e^{-\gamma \|x-y\|^2}$ (with $x,y\in\R^r$). Analogously, given an origin point $x_0 \in M$, the kernel
\begin{equation}
    \langle x, y \rangle_d^{x_0} := \frac{1}{2} \Big( d(x,x_0)^2 + d(y,x_0)^2 - d(x,y)^2  \Big),
\end{equation}
generalizes the linear kernel (with origin in $0$)
\begin{equation}
    \langle x, y \rangle = \frac{1}{2} \Big( \|x\|^2 + \|y\|^2 - \|x-y\|^2 \Big).
\end{equation}
However, this type of substitution does not always lead to a valid kernel. In fact, the class of distances such that $K_{\rm Gauss}$ and $\langle \cdot, \cdot \rangle_d^{x_0}$ are p.d. is completely characterised by the notion of Hilbertian (pseudo-)distance \citep{hein2005hilbertian}. A pseudo-distance $d$ is said to be Hilbertian if there exists a Hilbert space $\cF$ and a feature map $\Phi: \cX \to \cF$ such that $d(x,y) = ||\Phi(x)-\Phi(y)||_{\cF}$ for all $x,y \in M$. The functions $K_{\rm Gauss}$ and $\langle \cdot,\cdot \rangle_d^{x_0}$ are p.d. kernels if and only if the corresponding $d$ is Hilbertian. We refer to such kernels as \emph{distance substitution kernels} and we denote them as $K(d)$ (Proposition \ref{prop:app_d_subst} in the Appendix provides more details on distance substitution kernels).
\begin{remark}
    An Hilbertian pseudo-distance $d$ is a proper distance if and only if $\Phi$ is injective. 
\end{remark}
\begin{remark}
    For an Hilbertian distance $d$, $\langle x, y \rangle_d^{x_0} = \langle \Phi(x) - \Phi(x_0), \Phi(y) - \Phi(x_0) \rangle_{\cF}$.
\end{remark}
Hence defining p.d. kernels on $M=X$ boils down to find Hilbertian distances on $X$. Below we give two examples of such distances. 

\noindent {\bf MMD-based kernel.}~
For a bounded kernel $K_{\Omega}$ defined on the data space $\Omega_r$ and its RKHS $\cH_\Omega$, the Maximum Mean Discrepancy (MMD) pseudo-distance is for all $(\P,\Q) \in X^2$
\begin{equation}
    d_{\rm MMD}(\P, \Q) = \|\mu_{K_{\Omega}}(\P) - \mu_{K_{\Omega}}(\Q))\|_{\cH_{\Omega}},
\end{equation}
where $\mu_{K_{\Omega}}: \P \mapsto \int_{\Omega_r} K_{\Omega}(x,\cdot)d\P(x) \in \cH_{\Omega}$ is the kernel mean embedding operator \citep{berlinet2011reproducing}. $d_{\rm MMD}$ is by construction Hilbertian with $\cF = \cH_{\Omega} $ and $\Phi = \mu_{K_{\Omega}}$. Note that using MMD to build a kernel on $X$ requires two kernels, a kernel on $\Omega_r$ to build $d_{\rm MMD}$ and the outer kernel $K(d_{\rm MMD})$. Previous theoretical studies of distribution regression focused exclusively on MMD-based kernels \citep{szabo2016learning,fang2020optimal}. We will show that the established theory for MMD-based distribution regression can be extended to a larger class of distance substitution kernels.

\begin{remark}
    $d_{\rm MMD}$ is a proper distance if and only if $K_{\Omega}$ is characteristic, i.e. $\mu_{K_{\Omega}}$ is injective.
\end{remark}
\noindent {\bf Fourier Kernel.}~
Let $\P \in X$, its Fourier transform $F\left(\P\right)$ is defined by $F\left(\P\right)(t) = \int e^{i\langle x,t \rangle}d\P(x)$, $t \in \R^r$. In \citet{christmann2010universal} the distance
\begin{equation}
    d_F(\P,\Q) = \| F\left(\P\right) - F\left(\Q\right) \|_{L^2_{\mu}}
\end{equation}
was considered, where $\mu$ is a finite Borel measure with support over $\R^r$. $d_F$ is clearly Hilbertian and is a proper distance since the Fourier transform is injective. 

%\paragraph
\noindent {\bf Universality of Gaussian-like kernels.}~
%Gaussian-like kernels 
Kernels of the form \cref{eq:gauss_like} with both the MMD and the Fourier distances can be shown to be universal for $X$ equipped with the topology of the weak convergence in probability\footnote{$\left(\P_{n}\right)_{n \in \mathbb{N}} \in X^{\N}$ is said to converge weakly to $\P \in X$ if for any continuous function $f: \Omega_r \rightarrow \mathbb{R}$ (recall that $\Omega_r$ is compact), $\lim _{n \rightarrow+\infty} \int f \mathrm{~d} \P_{k}=\int f \mathrm{~d} \P.$}. 
\begin{definition}
  A continuous kernel $K$ on a compact metric space $(\cX,d_{\cX})$ is called universal if the RKHS $\cH_K$ of $K$ is dense in $C(\cX,d_{\cX})$, i.e. for every function $g: \cX \to \R$ continuous with respect to $d_{\cX}$ and all $\varepsilon>0$ there exists an $f \in \cH_K$ such that $\|f-g\|_{\infty} := \sup_{x \in X}|f(x) - g(x)| \leq \varepsilon$.
\end{definition}
Universality is an important property that indicates when a RKHS is powerful enough to approximate any continuous function with arbitrarily precision. We recall here a fundamental result regarding universality that we will use in the next section. 
\begin{theorem}[\citet{christmann2010universal}, Theorem 2.2] \label{th:steinwart}
    On a compact metric space $(\cX,d_{\cX})$ and for a continuous and injective map $\Phi: \cX \to \cF$ where $\cF$ is a separable Hilbert space, $K(x,y) = e^{-\gamma\|\Phi(x) - \Phi(x')\|_{\cF}^2}$ is universal.
\end{theorem}
For $\cX=X$ and $d_{P}$ the Prokhorov distance, which induced the topology describing the weak convergence of probability measures, \citet{christmann2010universal} uses this result to show that MMD (when the inner kernel is characteristic) and Fourier based Gaussian-like kernels are universal.

\section{Sliced Wasserstein Kernels} \label{section:SW}
We introduce the tools needed to define the Sliced Wasserstein kernels and give the first proof that SW Gaussian-like kernels are universal. Arising from the field of optimal transport \cite{villani2009optimal}, the Wasserstein metric is a natural distance to compare probability distributions.

%\paragraph
\noindent {\bf Wasserstein distance.}~ For $p \geq 1$, the $p-$Wasserstein distance on $X=\cM_+^1(\Omega_r)$ is
\begin{equation}
  d_{W_{p}}(\P, \Q)^p := \inf_{\pi \in \mathcal{U}(\P, \Q)} \int_{\R^r \times \R^r} \|x - y \|^p_2 \mathrm{d} \pi(x, y),
\end{equation}
where $\mathcal{U}(\P, \Q) := \{\pi \in X \otimes X \mid P^1_{\#} \pi=\P, P^2_{\#} \pi=\Q \}$ and
$P^{i}: \R^r \times \R^r \to \R^r, (x_1,x_2) \mapsto x_i$, $i=1,2$. 

\noindent {\bf One dimensional optimal transport.}~ 
The Wasserstein distance usually cannot be computed in closed form. For empirical distributions, it requires to solve a linear program that becomes significantly expensive in high-dimension. However, when $r=1$, the following closed form expression holds \citep[][Proposition 2.17]{santambrogio2015optimal}
\begin{equation}
    d_{W_{p}}(\P, \Q) = \|F_{\P}^{[-1]}-F_{\Q}^{[-1]} \|_{L^p(0,1)},
\end{equation}
where $L^p(0,1)$ if the set of $p-$integrable functions w.r.t the Lebesgue measure on $(0,1)$. If $\P$ and $\Q$ are empirical distributions, this can be handily computed by sorting the points and computing the average distance between the sorted samples; see \cref{sec:a_prac_sw} for more details. In what follows, we will either consider $p=1$ or $p=2$.

{\bf Sliced Wasserstein distance.}~ The closed-form expression of the one-dimensional Wasserstein distance can be lifted to $r>1$ using the mechanism of sliced divergences \citet{nadjahi2020statistical}. The Sliced Wasserstein (SW) distance is defined as,
\begin{equation} \label{eq:sw_p}
  d_{{\rm SW}_{p}}(\P, \Q)^p =\int_{\mathbb{S}^{r-1}}\hspace{-.15truecm} d_{W_{p}}{\left(\theta^*_{\#}\P, \theta^*_{\#}\Q \right)}^p d\theta.
\end{equation}
The distance averages the one-dimensional distances by projecting the data in every direction $\theta\in\mathbb{S}^{r-1}$. In practice, the integral over $S^{r-1}$ is approximated by sampling directions uniformly, see \cref{sec:a_prac_sw}. $d_{{\rm SW}_p}$ was introduced in \citet{bonneel2015sliced} in the context of Wasserstein barycenters and has been heavily used since as a computationally cheaper alternative to the Wasserstein distance. In addition to keeping a closed form expression, $d_{{\rm SW}_p}$ inherits useful topological properties from $d_{{\rm W}_p}$: it defines a distance on $\cM_+^1(\Omega_r)$ that implies the weak convergence of probability measures \citep{nadjahi2019asymptotic,nadjahi2020statistical}. More  importantly, $d_{{\rm SW}_1}$ and $d_{{\rm SW}_2}$ can be used to build a p.d. kernel on $\cM_+^1(\Omega_r)$, while it would not be possible with the original Wasserstein distance as the induced Wasserstein space has curvature \citet{feragen2015geodesic}. Indeed, we now show that $\sqrt{d_{{\rm SW}_1}}$ and $d_{{\rm SW}_2}$ are Hilbertian distances. Leveraging on the observations in \cref{sec:dist_subst_kernels} we conclude that $e^{-\gamma {\rm SW}_2^2(\cdot,\cdot)}$ and $e^{-\gamma {\rm SW}_1(\cdot,\cdot)}$ are valid kernels on probability measures.

{\bf ${\rm SW}_2$-kernel.}~ To show that $d_{{\rm SW}_2}$ is Hilbertian we build an explicit feature map and feature space. 

\begin{proposition}\label{prop:dsw2-is-hilbertian}
$d_{{\rm SW}_2}$ is Hilbertian. 
\end{proposition}

\begin{proof}
Taking $\cF = L^2\left((0,1) \times S^{r-1}, \lambda^1 \otimes \sigma^r\right)$ and 
\begin{equation} \label{eq:feature_sw2}
      \Phi: X \to \cF, \P \mapsto \left((t,\theta) \mapsto F_{\theta^*_{\#} \P}^{[-1]}(t) \right),
\end{equation}
we have,
\begin{equation*}
  \begin{aligned}
    d_{{\rm SW}_2}(\P,\Q)^2 &= \int_{S^{r-1}} \left\|F_{\theta^*_{\#} \P}^{[-1]}-F_{\theta^*_{\#} \Q}^{[-1]} \right\|_{L^2([0,1])}^2 d\theta \\ &=  \int_{S^{r-1}}\int_0^1 \left(F_{\theta^*_{\#}\P}^{[-1]}(t) -  F_{\theta^*_{\#}\Q}^{[-1]}(t) \right)^2 dt d\theta \\ 
    &= \|\Phi(\P) - \Phi(\Q) \|_{\cF}^2,
  \end{aligned}
\end{equation*}
as desired.
\end{proof}

Note that since we already know that $d_{{\rm SW}_2}$ is a proper distance, it implies that $\Phi$ is injective. 

%\paragraph
{\bf Comparison with \citet{kolouri2016sliced}.}~ Distance substitution kernels with $d_{{\rm SW}_2}$ were originally proposed in \citet{kolouri2016sliced} where it was first shown that such distance is Hilbertian. However, these results were limited to probability distributions admitting a density with respect to the Lebesgue measure. This was due to a different choice of feature map $\Phi$ in the proof of \cref{prop:dsw2-is-hilbertian} based on the Radon transform. In dimension one, when a probability $\P \in \cM_+^1(\Omega_1)$ admits a density $f = \frac{d\P}{d\lambda^1}$, the $2-$Wasserstein distance is
\begin{equation*}
  d_{W_{2}}(\P, \Q)^2 =  \int_{\R} (F_{\Q}^{[-1]}\circ F_{\P}(t) - t )^2 f(t)dt.
\end{equation*}
If $\P\in\cM_+^1(\Omega_r)$ admits a density, then for any direction $\theta\in\mathbb{S}^{r-1}$, $\theta^*_{\#}\P$ admits as density $\mathcal{R}f(.,\theta)$ on $\R$, with $\mathcal{R}$ the Radon transform. Then the SW distance is,
\begin{equation*}
  d_{{\rm SW}_{2}}(\P, \Q)^2 = \int (F_{\theta^*_{\#}\Q}^{[-1]}\circ F_{\theta^*_{\#}\P}(t) - t )^2 \mathcal{R}f(t,\theta)dtd\theta,
\end{equation*}
Showing that $d_{{\rm SW}_2}$ is Hilbertian using the above characterization is less direct than our approach since it requires introducing a reference measure $\P_0$ admitting a density. By defining 
\begin{equation*}
\Phi_{\P_0}(\P): (t,\theta) \mapsto \left(F_{\theta_{\#}^* \P}^{[-1]}\circ F_{\theta_{\#}^* \P_0}(t) - t \right)\sqrt{\mathcal{R}f(t, \theta)} 
\end{equation*}
it can be shown that $d_{{\rm SW}_{2}}$ is Hilbertian, namely ${ d_{{\rm SW}_{2}}(\P, \Q) = \|\Phi_{\P_0}(\P) - \Phi_{\P_0}(\Q) \|_{L^2(\R \times S^{r-1})}}$. In practice, this raises the question of how to choose the reference measure (and how it will impact the downstream learning algorithm). More importantly, the result in \cite{kolouri2016sliced} is only valid for absolutely continuous distributions and is not applicable when dealing with empirical distributions $\hat\P_n$ as in distribution regression (this would require a density estimation step that might be expensive and introduce additional errors). Our strategy avoids this. 

{\bf ${\rm SW}_1$-kernel.}~ For the ${\rm SW}_1$ distance we have the following. 

\begin{proposition}
$\sqrt{d_{{\rm SW}_1}}$ is Hilbertian.
\end{proposition}
$d_{{\rm SW}_1}$ is not Hilbertian while its squared root $\sqrt{d_{{\rm SW}_1}}$ is. This was originally observed in \citet{carriere2017sliced}, as the authors used it to build a kernel between persistence diagrams, but without a complete proof. For the sake of completeness, we provide a proof in \cref{sec:a_B}. The result hinges on Schoenberg’s theorem. Up to our knowledge, it is the first time that $K(d_{{\rm SW}_1})$ is used in the context of distribution regression. Note that the Gaussian-like kernel $K(\P,\Q) = \exp\left(-\gamma d_{{\rm SW}_1}(\P,\Q)\right)$ has no square in the exponential since it is the square root of the distance that is Hilbertian and not the distance itself. The result implies that there is some Hilbert space $\cF$ and feature map $\Phi: X \to \cF$ such that $\sqrt{d_{{\rm SW}_1}}(\P,\Q) = \|\Phi(\P) - \Phi(\Q)\|_{\cF}$. Interestingly, it implies that $\sqrt{d_{{\rm SW}_1}}$ is also a distance.  

{\bf Alternative Optimal Transport embeddings.}~ Linear Optimal Transportation (LOT) \citet{wang2013linear, merigot2020quantitative, moosmuller2020linear} defines an alternative transport-based Hilbertian distance. It requires the choice of a reference measure and a discretization scheme to be used in practice that add further complexity. Nonetheless, a separate study for LOT estimators would be an interesting future direction.

{\bf Algorithm.}~ We refer to kernels built on $X$ by substitution with either $d_{{\rm SW}_2}$ or $\sqrt{d_{{\rm SW}_1}}$ as \emph{Sliced-Wasserstein (SW) kernels}. Evaluating such kernels requires to evaluate \cref{eq:sw_p} for $p=1,2$. Instead of approximating the outer integral by sampling $\cU(S^{r-1})$ and evaluating exactly the inner integral we suggest to approximate both integrals as follows: 
% \begin{equation*}
%     \begin{aligned}
%         \hat d_{{\rm SW}_{p}}(\P, \Q)^p = \frac{1}{MN}\sum_{\substack{m=1 \\ \ell=1}}^{M,N} \left|F_{(\theta_m^*)_{\#}\P}^{[-1]}(t_{\ell}) -  F_{(\theta_m^*)_{\#}\Q}^{[-1]}(t_{\ell}) \right|^p,
%     \end{aligned}
% \end{equation*}
\begin{equation*}
    \begin{aligned}
        \hat d_{p}(\P, \Q)^p := \frac{1}{MN}\sum_{\substack{m=1 \\ \ell=1}}^{M,N} \left|F_{(\theta_m^*)_{\#}\P}^{[-1]}(t_{\ell}) -  F_{(\theta_m^*)_{\#}\Q}^{[-1]}(t_{\ell}) \right|^p,
    \end{aligned}
\end{equation*}
$\theta_1, \ldots, \theta_M \sim^{i.i.d} \cU(S^{r-1})$, $(t_1, \ldots, t_N) \sim^{i.i.d} \cU(0,1)$. This is akin to build a finite dimensional representation $\hat\Phi_{M,N}(\P) = \left(MN\right)^{-1/p}(F_{(\theta_m^*)_{\#}\P}^{[-1]}(t_{\ell}))_{m,\ell=1}^{M,N} \in \R^{MN}$ and  
% \begin{small}
\begin{equation} \label{eq:approx_dist}
    \hat d_{p}(\P,\Q) = \| \hat\Phi_{M,N}(\P) -  \hat\Phi_{M,N}(\Q) \|_{\ell_p(MN)}.
\end{equation}
% \end{small}
For empirical distributions $\hat\P_n = \frac{1}{n}\sum_{i=1}^n \delta_{x_i}$, we can make use of the closed form evaluation of $F_{\theta^*_{\#}\hat\P_n}^{[-1]}$, and we emphasize that the empirical distributions we evaluate are not required to have the same number of points (see \cref{sec:a_prac_sw}). The full procedure to evaluate $\hat\Phi_{M,N}(\hat\P_n)$ is summarized in \cref{alg:sw}. It can then be plugged in \cref{eq:approx_dist} and the approximated distance itself can be plugged in the distance substitution kernel of choice. 
\begin{algorithm}[tb]
   \caption{Evaluation of $\hat\Phi_{M,N}(\hat\P_n)$}
   \label{alg:sw}
\begin{algorithmic}
   \STATE {\bfseries Input:} data $\hat\P_n = \frac{1}{n}\sum_{i=1}^n \delta_{x_i} \in \cM_+^1(\Omega_r)$, $M, N \geq 1$
   \STATE {\bfseries Init:} $\Phi=0$ matrix of size $M$ by $N$ 
   \FOR{$m=1$ {\bfseries to} $M$}
   \STATE Sample $\theta_m \sim \cU(S^{d-1})$ 
   \STATE Sort sequence $x^{\theta_m}_i:= \langle \theta_m, x_i \rangle, i=1,\ldots,n$
   \STATE Label ordered sequence $x^{\theta_m}_{(1)} \leq \ldots \leq x^{\theta_m}_{(n)}$
   \FOR{$\ell=1$ {\bfseries to} $N$}
   \STATE Sample $t_{\ell} \sim \cU(0,1)$
   \STATE $\Phi_{m,\ell} = x_{(k)}^{\theta_m} \quad \text{ if } \quad t_{\ell} \in [\frac{k-1}{n},\frac{k}{n}), \quad 1 \leq k \leq n$
   \ENDFOR
   \ENDFOR
   \STATE {\bfseries Output:} Flatten$(\Phi) \in \R^{MN}$
\end{algorithmic}
\end{algorithm}

{\bf Universality of Gaussian SW kernels.}~
We exploit \cref{th:steinwart} to show that SW kernels are universal. 
\begin{proposition} \label{prop:univ_sw}
    $K_1:X^2 \to \R, (\P,\Q) \mapsto e^{-\gamma {\rm SW}_1(\P,\Q)}$ and $K_2:X^2 \to \R, (\P,\Q) \mapsto e^{-\gamma {\rm SW}_2^2(\P,\Q)}$ are universal for the topology induced by the weak convergence of probability measures.
\end{proposition}
\begin{proof}
    We prove the result for $K_2$, the proof for $K_1$ requires a slightly longer argument that is deferred to \cref{prop:a_uni_sw1}. Firstly, the compactness of $\Omega_r$ implies that $X$ is weakly compact \citep[Thm 6.4][]{parthasarathy2005probability}. Secondly, $d_{{\rm SW}_2}$ is Hilbertian with feature space given by $\cF = L^2\left((0,1) \times S^{r-1}, \lambda^1 \otimes \sigma^r\right)$, which is separable, and feature map given by \cref{eq:feature_sw2}. Since $d_{{\rm SW}_2}(\P,\Q) = \|\Phi(\P) - \Phi(\Q)\|_{\cF}$, $\Phi$ is continuous w.r.t $d_{{\rm SW}_2}$ and injective. Indeed $d_{{\rm SW}_2}$ satisfying the identity of indiscernibles implies that $\Phi$ is injective and $\Phi$ as a map from $(X,d_{{\rm SW}_2})$ to $(\cF, \|.\|_{\cF})$ is 1-Lipschitz, hence continuous. Finally, it has been shown in \citet{nadjahi2020statistical} that $d_{{\rm SW}_2}$ metrizes the weak convergence, hence continuity and compactness w.r.t $d_{{\rm SW}_2}$ is equivalent to weak continuity and weak compactness. By \cref{th:steinwart}, the universality of $K_2$ follows.  
\end{proof}

\section{Consistency in distribution regression with distance substitution kernels} \label{sec:excess_risk}
Up to our knowledge, every theoretical work on distribution regression has been focusing on MMD-based kernels. Here, we provide 
%offer 
the first analysis of distribution regression for any distance substitution kernel $K(d)$ on $X$ with known sample complexity for the chosen distance $d$. We follow the approach of \citet{fang2020optimal} which slightly improved the results in \citet{szabo2016learning}.  Using recent results on the sample complexity of the Sliced Wasserstein distance \citet{nadjahi2020statistical} we apply our excess risk bound to show that the KRR estimator with the SW kernel for distribution regression is consistent.

Given a Hilbertian distance $d$ (e.g. MMD or SW) on $X$ with feature map $\Phi: X \to \cF$, we consider a radial distance substitution kernel $K: X \times X \to \R, (\P,\Q) \mapsto q(d(\P,\Q)) = q(\|\Phi(\P) - \Phi(\Q)\|_{\cF})$ with $q: \R \to \R$. Valid functions $q$ are the one for which $K$ is a valid substitution kernel, for example $q:x \mapsto e^{-\gamma x^{2\beta}}$, $\gamma > 0, \beta \in [0,1]$ \cref{prop:app_d_subst}. The fact that $d$ is assumed to be a distance implies that $d(\P,\Q) = 0$ if and only if $\P = \Q$ and therefore $K(\P,\Q) = q(0)$ when $\P = \Q$. We denote by $\cH_K$ the induced RKHS and by $K_{\P} \in \cH_K$ the canonical feature map of $K$: $K(\P,\Q) = \langle K_{\P}, K_{\Q} \rangle_{\cH_K}$ for all $(\P,\Q) \in X^2$. Equipped with this kernel we consider the minimizer of the second stage risk \cref{eq:sol_KRR} and search for a bound on the excess risk $\cE(f_{\hat{\cD},\la}) - \cE(f_{\rho}) = \|f_{\hat{\cD},\la} - f_{\rho}\|_{L^2_{\rho_X}}^2$. 

In the context of MMD-based distribution regression \citet{szabo2016learning} provided high probability bounds on the excess risk while \citet{fang2020optimal} provided bounds on the expected excess risk with a slightly different analysis. The key assumption behind both approaches is that $K$ must be Hölder continuous with respect to the MMD. More precisely, that  there exists constants $L>0$ and $h \in (0,1]$ such that for all $(\P,\Q) \in X^2$ we have $\|K_{\P} - K_{\Q} \|_{\cH_K} \leq Ld_{\rm MMD}(\P,\Q)^h$. 
This assumption is instrumental to bound quantities that depend on the discrepancy between $\P$ and $\hat\P_n$ by the sample complexity of the MMD distance. Here, we advocate extending this assumption to a generic Hilbertian distance $d$, 
\begin{equation} \label{asm:holder_d_subst}
    \|K_{\P} - K_{\Q} \|_{\cH_K} \leq Ld(\P,\Q)^h.
\end{equation}
Let us note that \cref{asm:holder_d_subst} does not depend on $d$ but rather on the regularity of $q$ (see discussion in  \cref{subsec:a_holder} for more details). We now provide the main result of this section by first introducing the integral operator as a central tool in learning theory for kernel algorithms. 

{\bf Integral operator.}~ Since $X$ is compact, under the assumptions that $K$ is p.d., continuous and $\sup_{\P \in X} K(\P,\P) \leq \kappa^2$ (i.e. $K$ is a Mercer kernel \citet{aronszajn1950theory}), the integral operator $L_K: L_{\rho_X}^2 \to L_{\rho_X}^2$ defined by $$L_{K}(f)=\int_{X} K_{\P} f\left(\P\right) d\rho_{X}(\P), \quad f \in L_{\rho_{X}}^{2}$$
is compact, positive and self-adjoint. Furthermore $\cH_K$ is composed of continuous functions from $X$ to $\R$ and $L_K^{1/2}$ forms an isometry between $L^2_{\rho_X}$ and $\cH_K$: every function $f \in \cH_K$ can be written as $f = L_K^{1/2}g$ for some $g \in L^2_{\rho_X}$ and $\|f\|_{\cH_K} = \|g\|_{L^2_{\rho_X}}$ (section 1,2,3, chapter 3 \citet{Cucker02onthe}). 

\begin{theorem} \label{th:main}
Under the assumption that $\exists A > 0, |y| \leq A$, $\rho$ a.s., $f_{\rho} \in \cH_K$, $q$ is such that $K$ is a valid kernel, $K$ satisfies the Hölder condition \cref{asm:holder_d_subst}, $K$ is continuous and $\sup_{\P \in X} K(\P,\P) \leq \kappa^2$, if the expected sample complexity of $d$ is controlled such that $\E[d(\P,\hat\P_n)^2] \leq \alpha(n)$, $\alpha: \N \to \R_+$, $\P \in X$, then
\begin{equation*}
\begin{split}
  \E\|&f_{\hat{\cD},\lambda} - f_{\rho}\|_{L^2_{\rho_X}} \hspace{-0.15em}\leq  C\Bigg\{\left(\kappa\left\|f_{\rho}\right\|_{{\cH}_K}\hspace{-0.3em} +\hspace{-0.3em} \frac{A\mathcal{B}_{T, \lambda}^2}{\lambda}\hspace{-0.3em}+\hspace{-0.3em}\frac{A\mathcal{B}_{T, \lambda}}{\sqrt{\lambda}} \right) \\
  &\times \frac{L\alpha(n)^{h/2}}{\sqrt{\lambda}} \left(\sqrt{\kappa^2\frac{\mathcal{N}(\lambda)}{T\lambda}} + \frac{L\kappa\alpha(n)^{h/2}}{\lambda} +  1 \right)\\ &+\left(\frac{\mathcal{B}_{T, \lambda}}{\sqrt{\lambda}}+1\right)^{2} \frac{A}{\kappa} \mathcal{B}_{T, \lambda} + \left(\mathcal{B}_{T, \lambda}+\sqrt{\lambda}\right)\left\|f_{\rho}\right\|_{\cH_K}\Bigg\}
\end{split}
\end{equation*}
where $\mathcal{N}(\lambda):=\operatorname{Tr}\left(L_{K}\left(L_{K}+\lambda I\right)^{-1}\right)$, $\lambda > 0$ is the effective dimension, $\mathcal{B}_{T, \lambda}:=\frac{2 \kappa}{\sqrt{T}}\left(\frac{\kappa}{\sqrt{T \lambda}}+\sqrt{\mathcal{N}(\lambda)}\right)$ and $C$ is a constant the does not depend on any other quantity.
\end{theorem}
\begin{proof}
    We sketch the proof strategy here while \cref{subsec:a_th_proof} provides a complete proof. We start by decomposing
    \begin{equation*}
      \|f_{\hat{\cD},\la} - f_{\rho}\|_{L^2_{\rho_X}} \leq \|f_{\hat{\cD},\la} - f_{\cD,\la}\|_{L^2_{\rho_X}} + \|f_{\cD,\la} - f_{\rho}\|_{L^2_{\rho_X}}.
    \end{equation*}
    The term $\|f_{\cD,\la} - f_{\rho}\|_{L^2_{\rho_X}}$ corresponds to the excess risk of a KRR estimator that has access to the true input distributions (namely $\P$ rather than a sample $\hat\P_n$, see discussion in \cref{sec:dist_regr}). Bounds for such quantity have been thoroughly analysed in the literature \citep[see e.g.][]{caponnetto2007optimal}. This term is responsible for the last row in the bound of \cref{th:main}. The term $\|f_{\hat{\cD},\la} - f_{\cD,\la}\|_{L^2_{\rho_X}}$ is specific to the two-stage sampling nature of distribution regression and measures how much we loose by accessing each distribution through bag of samples. A leading term controlling this quantity  is $\|K_{\hat\P_{n}} - K_{\P_n} \|_{\cH_K}$, which is further bounded by the sample complexity of $d$ using the Hölder assumption in \cref{asm:holder_d_subst}.
\end{proof}
$\left\|f_{\rho}\right\|_{\cH_K}$ is an indicator of the regularity of $f_{\rho}$. The bound indicates that the excess risk goes down with the easiness of the learning problem that is reflected in the RKHS norm of $f_{\rho}$. Bounds as in \cref{th:main} are often refined using the so called \emph{source condition}: $\exists g_{\rho} \in L^2_{\rho_X}$ such that $f_{\rho} = L_K^{\epsilon}g_{\rho}$, $\epsilon > 0$. The larger $\epsilon$ the smoother is $f_{\rho}$ and for $\epsilon \geq 1/2$ we are in the well-specified setting, $f_{\rho} \in \cH_K$.

The effective dimension $\mathcal{N}(\lambda)$ is the other parameter controlling the bound and is used to describe the complexity of $\cH_K$. A common assumption is that $\mathcal{N}(\lambda) \leq c\lambda^{-\nu}$, $\nu \in (0,1]$ \citet{caponnetto2007optimal}. Note that, in the worst case scenario, this is always true with $\nu=1$.
\begin{equation*}
    \mathcal{N}(\lambda) = \sum_{i \geq 0} \frac{\lambda_i}{\lambda_i + \lambda} \leq \frac{Tr(L_K)}{\lambda} \leq \frac{\kappa^2}{\lambda}
\end{equation*} 
Plugging this bound and an explicit control over $\alpha(n)$ allows to get more explicit rates.  
\begin{corollary} \label{cor:simplified_bound}
    Under the assumptions of \cref{th:main}, for $\alpha(n) = c n^{-\beta}$, $\beta, c > 0$, taking $\lambda = \max(\frac{1}{\sqrt{T}}, \frac{1}{n^{h\beta/2}})$ the bound can be simplified as 
    \begin{equation*}
        \E\|f_{\hat{\cD},\lambda} - f_{\rho}\|_{L^2_{\rho_X}}
        \leq C\left(\frac{1}{\sqrt[4]{T}} + \frac{1}{n^{h\beta/4}} \right)\left(\left\|f_{\rho}\right\|_{\cH_K} \hspace{-0.15em}+\hspace{-0.15em} 1\right)
    \end{equation*}
    where $C$ is a constant that depends only on $A$, $\kappa$ and $L$.
\end{corollary}
We defer the proof to \cref{subsec:a_simplified_bound}. For $d_{\rm MMD}$ it is known that $\beta=1$, for completeness the proof is given in \cref{subsec:a_sample_mmd}. Plugging this rate in \cref{cor:simplified_bound}, for the Gaussian-like kernel for example (h=1), leads to the same result as presented in \citet{fang2020optimal}, the excess risk is in $\frac{1}{\sqrt[4]{T}} + \frac{1}{\sqrt[4]{n}}$. For SW, it is shown in \citet{nadjahi2020statistical} that both $\E[d_{{\rm SW}_2}(\P, \hat\P_n)^{2}]$ and $\E[d_{{\rm SW}_1}(\P, \hat\P_n)]$ are in $O(n^{-1/2})$, hence $\beta=1/2$. Leading for $h=1$ to an expected excess risk in the order of $\frac{1}{\sqrt[4]{T}} + \frac{1}{\sqrt[8]{n}}$. Due to its slower sample complexity, the bound appears worse for {\rm SW} than for MMD, when we neglect $\left\|f_{\rho}\right\|_{\cH_K}$. However, such term captures the regularity of the solution $f_\rho$ in terms of the underlying metric used to compare probability distribution and can potentially differ significantly between MMD and SW-based kernels. In particular, we argue that when $f_\rho$ needs to be sensitive to geometric properties of distributions in $\cM_+^1(\Omega_r)$, then $\left\|f_{\rho}\right\|_{\cH_K}$ might be significantly smaller for $\mathcal{H}_K$ the space associated to a SW kernel rather than the MMD based one. In the following, we provide empirical experiments supporting this perspective, showing that SW kernels are often superior to MMD ones in real and synthetic applications. We postpone the theoretical investigation of this important question to future work. 

\section{Experiments}
\label{sec:exp}
In this section we compare the numerical performance of SW kernels versus MMD kernels on two tasks. First, we look at the synthetic task that consists of counting the number of components of Gaussian mixture distributions. Then, we present results for distribution regression on the MNIST and Fashion MNIST datasets.

\begin{remark}
    In both tasks, considered as classification as the labels are discrete, we use KRR hence the mean squared error loss. According to the theory of surrogate methods \citet{bartlett2006convexity}, the squared loss is ``classification calibrated'' (by using 1-hot encoding and sign or argmax decoding), namely it is a valid surrogate loss for classification problems, similarly to hinge and logistic losses (see e.g. \citet{mroueh2012multiclass} and \citet{meanti2020kernel} for theoretical and empirical analyses, respectively).  In particular, excess risk bounds for KRR (e.g. our \cref{th:main}) automatically yield bounds for classification (see [2]).
\end{remark}

{\bf Synthetic Mode classification.}~ Taken from  \citet{oliva2014fast}, we consider the problem of counting the number of modes in a Gaussian Mixture Model (GMM). This can be seen as a model selection process where given a set of points we want to fit a GMM but first we learn a mapping from the set of points to the number of components $p$. The data are simulated as follows. Let $T$ be the number of GMM distributions, $n$ the number of points drawn from each GMM, $C$ the maximum number of components and $r$ the input dimension. For each $t=1,\ldots,T$ we sample independently $p \sim \cU(\{1,\ldots,C\})$, $\mu_1,\ldots,\mu_p \sim^{i.i.d} \cU([-5,5]^r)$, and $\Sigma_j = a_jA_jA_j^{\top} + B_j, j=1,\ldots,p$ where $a_j \sim \cU([1,4])$, $A_j$ is a $r\times r$ matrix with entries sampled from $\cU([-1,1])$ and $B_j$ is a diagonal matrix with entries sampled from  $\cU([0,1])$. \cref{fig:gmm} presents examples of GMM distributions sampled as above for $r=2$, $C=3$ and $n=100$. We compare the Gaussian ${\rm SW}_1$ and ${\rm SW}_2$ kernels (with $M=100$, $N=100$ \cref{eq:approx_dist}) against the doubly Gaussian MMD kernel. In \cref{table-gmm} we report the RMSE for different configurations of $T, n, C$ and $r$. A validation set of size $50$ is used to select the bandwidth of the Gaussian kernels and the regularisation parameter $\lambda$ (see \cref{sec:exp_details} for details) and a test set of size 100 is used for the results. Each experiment is repeated 5 times.

%\todo{MP: in the main body we may present only SW2 (or the best of the two) and report the full table in the appendix}
\begin{table}[t]
\caption{Test RMSE and standard deviation for the mode classification task with different configuration for $T$, $n$, $C$ and $r$ (tested $5$ times each). Comparison of a doubly Gaussian MMD kernel and Gaussian ${\rm SW}_2$ and ${\rm SW}_1$ kernels.}
\label{table-gmm}
\vskip 0.15in
\begin{center}
\begin{scriptsize}
\begin{sc}
\begin{tabular}{cccc|ccc}
  \toprule
  $T$ & $n$ & $C$ & $r$ & ${\rm MMD}$& ${\rm SW}_2$ & ${\rm SW}_1$  \\
  \midrule
  $100$ & $50$ & $2$ & $2$ & $0.64 ~(0.09)$ & $0.48 ~(0.03)$ & $\mathbf{0.47 ~(0.06)}$  \\
  $100$ & $50$ & $10$ & $2$ & $3.66 ~(0.88)$ & $\mathbf{2.92 ~(0.32)}$ & $3.01 ~(0.24)$  \\
  $100$ & $50$ & $2$ & $10$ & $0.68 ~(0.02)$ & $0.54 ~(0.05)$ & $\mathbf{0.53 ~(0.05)}$  \\
  $100$ & $50$ & $10$ & $10$ & $3.48 ~(0.95)$ & $\mathbf{3.17 ~(0.32)}$ & $3.19 ~(0.14)$  \\
  $500$ & $50$ & $2$ & $2$ & $0.66 ~(0.08)$ & $0.44 ~(0.07)$ & $\mathbf{0.42 ~(0.07)}$  \\
  $500$ & $50$ & $10$ & $2$ & $4.1 ~(1.03)$ & $\mathbf{2.63 ~(0.27)}$ & $2.68 ~(0.24)$  \\
  $500$ & $50$ & $2$ & $10$ & $0.64 ~(0.13)$ & $0.44 ~(0.08)$ & $\mathbf{0.42 ~(0.04)}$  \\
  $500$ & $50$ & $10$ & $10$ & $3.72 ~(1.14)$ & $\mathbf{3.07 ~(0.42)}$ & $3.08 ~(0.35)$  \\
  $500$ & $250$ & $2$ & $2$ & $0.64 ~(0.1)$ & $\mathbf{0.39 ~(0.03)}$ & $0.42 ~(0.04)$  \\
  $500$ & $250$ & $10$ & $2$ & $3.8 ~(0.1)$ & $\mathbf{2.4 ~(0.15)}$ & $2.41 ~(0.16)$  \\
  $500$ & $250$ & $2$ & $10$ & $0.65 ~(0.15)$ & $\mathbf{0.39 ~(0.01)}$ & $0.41 ~(0.03)$  \\
  $500$ & $250$ & $10$ & $10$ & $3.64 ~(0.83)$ & $2.84 ~(0.17)$ & $\mathbf{2.8 ~(0.31)}$  \\
  \bottomrule
\end{tabular}
\end{sc}
\end{scriptsize}
\end{center}
\vskip -0.1in
\end{table}

{\bf MNIST Classification.}~  
MNIST \cite{lecun2010mnist} and Fashion MNIST \cite{xiao2017fashion} are constituted of gray images of size
28 by 28. Classes in MNIST are the digits from 0 to 9 and classes in Fashion MNIST are 10 categories of clothes. SW and MMD kernels take as inputs probability distributions, we therefore convert each image to an histogram. For an image $t$, let $(x^t_i,
y^t_i)_{i=1}^{n_t}$ denotes the positions of the active pixels (i.e. with an
intensity greater than $0$), $n_t$ is the number of active pixels. The positions
are renormalised to be in the grid $[-1,1]^2$. We therefore have $\Omega_2 =
[-1,1]^2$. Let $(I^t_i)_{i=1}^{n_t} \in \{0, \ldots 255\}$ be the associated pixel
intensities. We consider the weighted histograms
$\hat\P_t = \sum_{i=1}^{n_t}c^t_i \delta_{\{x^t_i, y^t_i\}} \in \cM_+^1([-1,1]^2)$ where $c^t_i = I^t_i/\sum_{j=1}^{n_t}I^t_j$. We consider Gaussian distance substitution kernels \cref{eq:gauss_like} for $d_{\rm MMD}$, $d_{{\rm SW}_2}$ and $\sqrt{d_{{\rm SW}_1}}$ (with $M=100$, $N=100$ \cref{eq:approx_dist}). The kernel chosen to build $d_{\rm MMD}$ is also Gaussian. For comparison we also apply an Euclidean Gaussian kernel on the flatten images. We consider the mean squared error loss and build KRR estimators \cref{eq:sol_KRR}. We use a train set of 1000 images, a validation set of 300 images and a test set of 500 images to evaluate the estimators. Each set is balanced. We use the train-validation split to select the bandwidth of the Gaussian kernels and the regularisation parameter $\lambda$ (see \cref{sec:exp_details} for details). In order to make the task more challenging we test the kernels in three configurations: on the raw images and on images randomly rotated (with a maximum angle of either $\pm 15$ or $\pm 30$ degrees) and translated, \cref{fig:mnist-var}. We run the experiment 5 times on different part of the MNIST and Fashion MNIST datasets. Results are given in \cref{table-mnist}. 

{\bf Comparison.}~  
SW kernels consistently beat MMD kernels in both experiments, which indicates that the SW distance is a better measure of similarity between histograms than the MMD distance. On the Fashion MNIST dataset for which the images are very well centered both MMD and SW are beaten by the standard Gaussian kernel on the flatten images. However, this kernel is less robust and its performance deteriorates far quicker than SW kernels when we perturb the images, the MMD kernel is even less robust. While the focus of this work is to suggest an alternative to MMD-based distribution regression, in \cref{sec:additional_exp} we also make a comparison to KRR with the square root total variation and Hellinger distances, which are both c.n.d.

\begin{figure}[t]
\vskip 0.2in
\begin{center}
\centerline{\includegraphics[width=0.80\columnwidth]{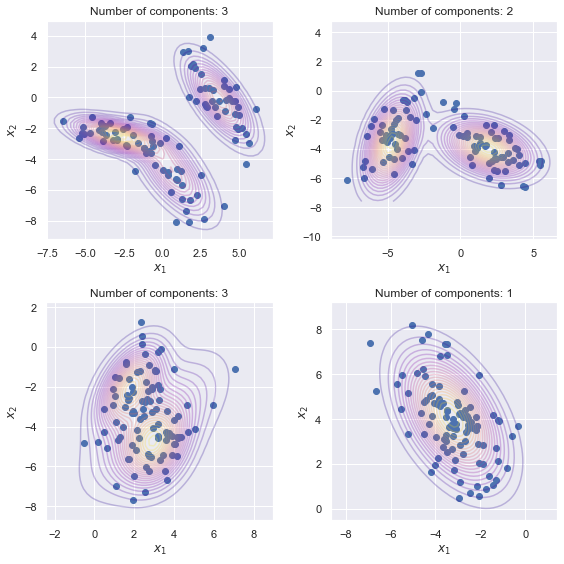}}
\caption{Samples and level curves for sample GMMs in the mode classification experiment, with $r=2$, $C=3$ and $n=100$.}
\label{fig:gmm}
\end{center}
\vskip -0.2in
\end{figure}

\begin{figure}[t]
\vskip 0.2in
\begin{center}
\centerline{\includegraphics[width=0.80\columnwidth]{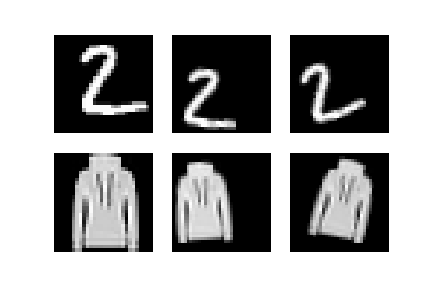}}
\caption{Sample transformations applied to MNIST (Top row) and Fashion MNIST (Bottom row) images. (Left column) raw $28 \times 28$ pixel images. Images are padded to $34 \times 34$ pixels and randomly roto-translated by degree drawn uniformly (for each image) on $[-15,15]$ (Middle column) and $[-30,30]$ (Right column.}
\label{fig:mnist-var}
\end{center}
\vskip -0.2in
\end{figure}

\begin{table}[t]
\caption{Mean accuracy and standard deviation on rotated MNIST and Fashion MNIST with maximum rotation $\theta_\vee = 0, \pi/12, \pi/6$ respectively. Comparison between KRR estimators with standard Gaussian kernel (RBF), doubly Gaussian MMD kernel (MMD), a Gaussian ${\rm SW}_2$ kernel and a Gaussian ${\rm SW}_1$. 
}
\label{table-mnist}
\vskip 0.15in
\begin{center}
\begin{scriptsize}
\begin{sc}
\begin{tabular}{lcccc}
  \toprule
   MNIST & RBF & MMD  & ${\rm SW}_2$ & ${\rm SW}_1$  \\
  \midrule
  Raw & $0.9 ~(0.01) $ & $0.79 ~(0.03)$ & $\mathbf{0.93 ~(0.01)}$ & $0.91 ~(0.03)$  \\
  $\theta_{\vee}=\frac{\pi}{12}$ & $0.51 ~(0.03)$ & $0.40 ~(0.05)$ & $\mathbf{0.85 ~(0.02)}$ & $0.66 ~(0.02)$   \\
  $\theta_{\vee}=\frac{\pi}{6}$ & $0.47 ~(0.03)$ & $0.34 ~(0.04)$ & $\mathbf{0.82 ~(0.01)}$ & $0.61 ~(0.01)$ \\
  \bottomrule
\end{tabular}
\end{sc}
\end{scriptsize}
\end{center}
\begin{center}
\begin{scriptsize}
\begin{sc}
\begin{tabular}{lcccc}
  \toprule
  Fashion & RBF & MMD & ${\rm SW}_2$ & ${\rm SW}_1$ \\
  \midrule
  Raw & $\mathbf{0.82~( 0.01)}$ & $0.75~(0.01)$ & $0.81~(0.01)$ & $0.77~(0.01)$   \\
  $\theta_{\vee}=\frac{\pi}{12}$ & $0.66~(0.01)$ & $0.48~(0.03)$ & $\mathbf{0.74~(0.01)}$ & $0.60~(0.02)$\\
  $\theta_{\vee}=\frac{\pi}{6}$ & $0.64~(0.01)$ & $0.45~ (0.01$) & $\mathbf{0.70~(0.01)}$ & $0.58~(0.01$) \\
  \bottomrule
\end{tabular}
\end{sc}
\end{scriptsize}
\end{center}
\vskip -0.1in
\end{table}

\section{Conclusion}

The main goal of this work was to address distribution regression from the perspective of Optimal Transport (OT), in contrast to the more common and well-established MMD-based strategies. For this purpose we proposed the first kernel ridge regression estimator for distribution regression, using Sliced Wasserstain distances to define a distance substitution kernel. We proved finite learning bounds for these estimators, extending previous work from \citet{fang2020optimal}. We empirically evaluated our estimator on a number of real and simulations experiments, showing that tackling distribution regression with OT-based methods outperforms MMD-based approaches in several settings.

We identify two main directions for future research: firstly, we plan to better identify which regimes are more favorable for OT-based and MMD-based distribution regression. This is an open question in the literature and requires delving deeper in the analysis of the approximation properties of the distance substitution kernels induced by the two distances respectively. Secondly, our theoretical results, do not account for the fact that the Sliced Wasserstein distance needs to be approximated via Monte-Carlo methods. This analysis was beyond the scope of this work, which focuses on comparing OT and MMD based distribution regression. However, we plan to account for this further (third) sampling stage in future work. 

% Acknowledgements should only appear in the accepted version.
\section*{Acknowledgements}
We would like to thank Saverio Salzo for helpful discussions. Most of this work was done when D.M. was a
research assistant at the Italian Institute of Technology. D.M. is supported by the PhD fellowship from the Gatsby Charitable Foundation and the Wellcome Trust. D.M. is also an ELLIS PhD student. C.C. acknowledges the support of the Royal Society (grant SPREM RGS\textbackslash{}R1\textbackslash{}201149) and Amazon.com Inc. (Amazon
Research Award – ARA).

\section*{Software and Data}
Code is available at \url{https://github.com/DimSum2k/DRSWK}.

% If a paper is accepted, we strongly encourage the publication of software and data with the
% camera-ready version of the paper whenever appropriate. This can be
% done by including a URL in the camera-ready copy. 

\bibliography{DRSWK}
\bibliographystyle{icml2022}

%%%%%%%%%%%%%%%%%%%%%%%%%%%%%%%%%%%%%%%%%%%%%%%%%%%%%%%%%%%%%%%%%%%%%%%%%%%%%%%
%%%%%%%%%%%%%%%%%%%%%%%%%%%%%%%%%%%%%%%%%%%%%%%%%%%%%%%%%%%%%%%%%%%%%%%%%%%%%%%
% APPENDIX
%%%%%%%%%%%%%%%%%%%%%%%%%%%%%%%%%%%%%%%%%%%%%%%%%%%%%%%%%%%%%%%%%%%%%%%%%%%%%%%
%%%%%%%%%%%%%%%%%%%%%%%%%%%%%%%%%%%%%%%%%%%%%%%%%%%%%%%%%%%%%%%%%%%%%%%%%%%%%%%
\newpage
\appendix
\onecolumn
{\huge \textbf{Supplementary Material}}

\vspace{.3truecm}
Below we give an overview of the structure of the supplementary material.

\begin{itemize}
    \item  Appendix \ref{sec:a_A} contains two derivations related to the distribution regression set up: the decomposition of the excess risk and the closed-form expression of the minimizer of the second-stage empirical risk.  
    \item In Appendix \ref{sec:a_B} 
    we present more examples of distance substitution kernels and the proof that $\sqrt{d_{{\rm SW}_1}}$ is Hilbertian. 
    \item Appendix \ref{sec:a_C} contains the proof of \cref{prop:univ_sw} for ${\rm SW}_1$.
    \item In Appendix \ref{sec:a_prac_sw}  we show how the Sliced Wasserstein distance can be computed and approximated in practice. 
    \item Appendix \ref{sec:a_main} contains a discussion around the Hölder assumption and shows that it holds for $q: x \mapsto e^{-\gamma x^{2\beta}}$. Specifically,  \cref{subsec:a_th_proof} contains the proof of \cref{th:main} and \cref{subsec:a_sample_mmd} contains the proof of \cref{cor:simplified_bound}.
    \item Appendix \ref{sec:exp_details} contains the details of the choice of hyperparameters for the experimental section. 
    \item Appendix \ref{sec:additional_exp} contains additional experiments with the Hellinger and total variation distances.
\end{itemize}

\section{Distribution Kernel Ridge Regression} \label{sec:a_A}
\begin{proposition} \label{prop:a_bayes_f}
    For all $f: X \rightarrow \R$ measurable, $$\cE(f) := \int_{X \times Y}(f(\P) - y)^2 d\rho(\P,y) = \|f - f_{\rho}\|_{L^2_{\rho_X}}^2 + \cE(f_{\rho}).$$
    The second term is independent of $f$ and the first term is minimised for $f = f_{\rho}$ a.e. which implies
    $$f_{\rho}(\P) = \argmin_f \cE(f) = \argmin_f \|f - f_{\rho}\|_{L^2_{\rho_X}}.$$
\end{proposition}
\begin{proof}
\begin{equation*}
\begin{aligned}
     \int_{X \times Y} (f(\P) - y)^2 d\rho(\P,y) &= \int_{X \times Y} (f(\P) - f_{\rho}(\P) + f_{\rho}(\P) - y)^2 d\rho(\P,y) \\ &= \int_X (f(\P) - f_{\rho}(\P))^2 d\rho_X(\P) + 2\int_{X \times Y} (f(\P) - f_{\rho}(\P))(f_{\rho}(\P) - y) d\rho(\P,y) + \cE(f_{\rho})  \\ &= \|f - f_{\rho}\|_{L^2_{\rho_X}}^2 + 2\int_X (f(\P) - f_{\rho}(\P))(f_{\rho}(\P) - \underbrace{\int_Y y d\rho_{Y \mid X}(y \mid \P)}_{=f_{\rho(\P)}})d\rho_X(\P) + \cE(f_{\rho})  \\ &= \|f - f_{\rho}\|_{L^2_{\rho_X}}^2 + \cE(f_{\rho})
\end{aligned}
\end{equation*}
\end{proof}

The following proposition shows how to derive the explicit form of the minimizer of the second-stage empirical risk \cref{eq:risk2}.  
\begin{proposition} \label{prop:a_KRR}
Let $K: X \times X \to \R$ be a p.d. kernel with RKHS denoted as $\cH_K$, then
\begin{equation} 
    f_{\hat{\cD},\lambda} = \argmin_{f \in \cH_K}\frac{1}{T}\sum_{t=1}^T \left(f\left(\hat{\P}_{t,n_{t}}\right) - y_t\right)^2 + \lambda \|f\|_{\cH_K}^2
\end{equation}
takes the form
\begin{equation*} 
    f_{\hat{\cD},\lambda}(\P) = (y_1, \ldots, y_T)(K_T + \lambda T I_T)^{-1}k_{\P}, \quad \forall \P \in X
\end{equation*}
where $[K_T]_{t,l} = K\left(\hat{\P}_{t,n_{t}},\hat{\P}_{l,n_{l}}\right) \in \R^{T \times T}$ and $k_{\P} = \left(K\left(\P, \hat{\P}_{1,n_{1}}\right), \ldots, K\left(\P, \hat{\P}_{T,n_{T}}\right)\right)^{\top} \in \R^T$.
\end{proposition}
\begin{proof}
By the representer theorem \citet{scholkopf2001generalized} a minimizer of the empirical risk takes the form $f_{\alpha}(\cdot) = \sum_{t=1}^T \alpha_t K(\hat\P_{t,n_{t}}, \cdot)$. The minimisation of the empirical risk is therefore equivalent to 
\begin{equation*}
    \begin{aligned}
         \alpha^* &= \argmin_{\alpha \in \R^T} \|K_T\alpha - y\|_{\R^T}^2 +  \lambda T \alpha^{\top} K_T \alpha
         \\ &= \argmin_{\alpha \in \R^T} \alpha^T K_T\left(K_T + \lambda T I_T \right) \alpha - 2y^{\top}K_T\alpha.
    \end{aligned}
\end{equation*}
Taking the gradient with respect to $\alpha$, we obtain the first order condition $K_T\left[(K_T + \lambda T I_T)\alpha - y \right] = 0$. Hence any solution takes the form
$$
\alpha'=\underbrace{(K_T+\lambda T I_T)^{-1}y}_{=: \alpha^*}+\epsilon, \text { with } K_T\epsilon=0.
$$
When $K_T$ is singular we have an infinity of solutions, however, they all lead to the same solution in $\cH_K$. Indeed, if $\alpha'=\alpha^*+\epsilon$ with $K_T \epsilon=0$, then:
$$
\left\|f_{\alpha'}-f_{\alpha^*}\right\|_{\mathcal{H}_K}^{2}=\left(\alpha'-\alpha^*\right)^{\top} K_T\left(\alpha'-\alpha^*\right)=0
$$
therefore $f_{\alpha'}=f_{\alpha^*}$. It shows that Kernel Ridge Regression has a unique solution $f^* \in \cH_K$ defined by
\begin{equation*} 
    f^*(\P) = (y_1, \ldots, y_T)(K_T + \lambda T I_T)^{-1}k_{\P}, \quad \left(\P \in X\right)
\end{equation*}
which can potentially be expressed by several $\alpha$'s if $K_T$ is singular.
\end{proof}

\section{Distance substitution kernels} \label{sec:a_B}
\begin{proposition}[Proposition 1 \cite{haasdonk2004learning}] \label{prop:app_d_subst} For a set $M$ and a pseudo-distance distance $d$ on $M$, the following statements are equivalent:
    \begin{itemize}
    \item $d$ is a Hilbertian pseudo-distance;
    \item $K_{lin}(x,y) = \langle x, y \rangle_d^{x_0}$, for all $x_0 \in M$, $(x,y) \in M^2$ is p.d.;
    \item $K_{poly}(x,y) = (c +  \langle x, y \rangle_d^{x_0})^{\ell}$ for all $c
      \geq 0$ and $\ell \in \N$, $(x,y) \in M^2$ is p.d.;
    \item $K(x,y) = \exp(-\gamma d^{2\beta}(x,y))$ for all $\gamma \geq 0,
      \beta \in [0,1]$, $(x,y) \in M^2$ is p.d.;
    \end{itemize}
\end{proposition}
We can therefore embed Hilbertian pseudo-distances in polynomial-like, Laplace-like ($\beta=1/2$) and Gaussian-like ($\beta=1$) kernels. Furthermore \citet{szabo2016learning} provides examples of Cauchy-like, t-student-like and inverse multiquadratic-like kernels. 

We now prove that $\sqrt{d_{{\rm SW}_1}}$ is Hilbertian by showing that $d_{{\rm SW}_1}$ is conditionally negative definite (c.n.d). We first recall the definition of a c.n.d function and how it relates to Hilbertian pseudo-distances. 
\begin{definition}
   Let $M$ be a set. A map $f:M \times M \rightarrow \R$ is called {\em conditionally negative definite} (c.n.d.) if and only if it is symmetric and satisfies:
   $$
   \sum_{i,j=1}^{n} a_{i} a_{j} f(x_{i},x_{j}) \leq 0
   $$
   for any $n \in \N$, $(x_1, \ldots, x_n) \in M^{n}$ and $(a_1, \ldots, a_n) \in \R^{n}$ with $\sum_{i=1}^{n} a_{i} = 0$.
\end{definition}
Square of Hilbertian pseudo-distances are c.n.d.
\begin{proposition}
    Let $\Phi:M \to \cF$ be a mapping to a Hilbert space $\cF$ and $f(x,y)=\|\Phi(x) - \Phi(y)\|^{2}_{\cF}$, then $f$ is c.n.d.
\end{proposition}
\begin{proof}
    $f$ is symmetric and for $n \in \N$, $(a_1, \dots, a_n) \in \R^n$, $(x_1, \dots, x_n) \in M^n$, we have:
    \begin{equation*}
        \begin{aligned}
        \sum_{1 \leq i, j \leq n} a_i a_j f(x_i, x_j) &= \sum_{1 \leq i, j \leq n} a_i a_j\left( \|\Phi(x_i)\|^2 + \|\Phi(x_j)\|^2 - 2 \langle \Phi(x_i), \Phi(x_j)\rangle \right) \\
        &= 2\left(\sum_{i=1}^n a_i \|\Phi(x_i)\|^2\right) \left(\sum_{j=1}^n a_j\right) - 2 \left\|\sum_{i=1}^n a_i \Phi(x_i)\right\|^2
        \end{aligned}
    \end{equation*}
    Under the constraint $\sum_{i=1}^n a_i = 0$, the
    first term disappears and only a nonpositive term remains. Hence, $f: (x, y) \mapsto \|\Phi(x) - \Phi(y)\|^2$ is c.n.d.
\end{proof}
A corollary of Schoenberg’s theorem (see chapter 3 in \cite{berg84_harmon}) shows that the converse holds. 
\begin{proposition} \label{prop:a_shoen}
  Let $f$ be a c.n.d. function on $M$ such that $f(x,x)=0$ for any $x \in M$. Then, there exists a Hilbert space $\cF$ and a mapping $\Phi:M \rightarrow \cF$ such that for any $(x,y) \in M^2$,
  $$
  f(x,y) = ||\Phi(x)-\Phi(y)||^{2}_{\cF}.
  $$
\end{proposition} 
As a consequence showing that $\sqrt{d_{{\rm SW}_1}}$ is Hilbertian is equivalent to showing that $d_{{\rm SW}_1}$ is c.n.d. We start with the following lemma. 
\begin{lemma}
    Let $M$ be a set, $(\Omega,\mathcal{A},\mu)$ a measurable space and $\phi: M
    \times \Omega \to \R$ be such that $\phi(x,\cdot) \in L_1(\mu)$ for all $x
    \in M$. Then, $f: M \times M \to \R, (x,y) \mapsto
    \int_{\Omega}|\phi(x,\omega)-\phi(y,\omega)|d\mu(\omega)$ is c.n.d.
\end{lemma}
\begin{proof}
    First, notice that for all $(u,v) \in [a, +\infty)$, $\displaystyle \min(u,v)  = \int_a^{+\infty} 1_{t \leq
      u}1_{t \leq v}dt + a$. Secondly, for all $u,v \in
    \R$, $|u - v| = u + v -2\min(u,v)$. For all $n \in \N^*$, $(x_1,\ldots,x_n)
    \in M^n$ and $\omega \in \Omega$ define $a_n^{\omega} := a(\omega;
    x_1,\ldots,x_n) := \min_{i=1,\ldots,n} \phi(x_i,\omega)$. For all $(a_1, \ldots,
    a_n) \in \R^n$ such that $\sum_{i=1}^n a_i = 0$,
    \begin{equation*}
      \begin{aligned}
        \sum_{i,j=1}^n a_i a_j f(x_i,x_j) &= \int_{\Omega} \sum_{i,j=1}^n a_i a_j \left\{\phi(x_i,\omega) + \phi(x_j,\omega) - 2\min(\phi(x_i,\omega),\phi(x_j,\omega)) \right\} d\mu(\omega) \\
        &= -2 \int_{\Omega} \sum_{i,j=1}^n a_i a_j \min(\phi(x_i,\omega),\phi(x_j,\omega))d\mu(\omega) \qquad \text{ since } \sum_{i=1}^n a_i = 0 \\
        &= -2 \int_{\Omega} \sum_{i,j=1}^n a_i a_j \left(\int_{a_n^{\omega}}^{+\infty} 1_{t \leq \phi(x_i,\omega)} 1_{t \leq \phi(x_j,\omega)}dt + a_n^{\omega}\right) d\mu(\omega) \\
        &=  -2 \int_{\Omega}\int_{a_n^{\omega}}^{+\infty} \left(\sum_{i=1}^n a_i 1_{t \leq \phi(x_i,\omega)} \right)^2 dt + \left(\sum_{i=1}^n a_i \right)^2a_n^{\omega} d\mu(\omega) \\ 
        &=  -2 \int_{\Omega}\int_{a_n^{\omega}}^{+\infty} \left(\sum_{i=1}^n a_i 1_{t \leq \phi(x_i,\omega)} \right)^2 dt d\mu(\omega) \qquad  \text{ since } \sum_{i=1}^n a_i = 0 \\
        & \leq 0 
      \end{aligned}
    \end{equation*}
    which proves that $f$ is c.n.d.
\end{proof}
Applying the result to $d_{W_{1}}(\P, \Q) = \int_0^1 |F_{\P}^{[-1]}(x)-F_{\Q}^{[-1]}(x)|dx$ shows that $d_{W_{1}}$ is c.n.d. We conclude that $d_{{\rm SW}_{1}}$ is c.n.d by linearity of the integral. 
\begin{proposition} $d_{{\rm SW}_{1}}(\P, \Q) =\int_{\mathbb{S}^{r-1}} d_{W_{1}}\left(\theta^*_{\#}\P, \theta^*_{\#}\Q \right) d\theta$ is c.n.d.
\end{proposition}
\begin{proof}
    For all $n \in \N$, $(a_1, \dots, a_n) \in \R^n$ such that $\sum_{i=1}^n a_i = 0$ and $(\P_1, \dots, \P_n) \in X^n$,
    \begin{equation*}
        \sum_{1 \leq i, j \leq n} a_i a_j d_{{\rm SW}_{1}}(\P_i, \P_j) = \int_{\mathbb{S}^{r-1}} \underbrace{\sum_{1 \leq i, j \leq n} a_i a_j d_{W_{1}}\left(\theta^*_{\#}\P_i, \theta^*_{\#}\P_j \right)}_{\leq 0} d\theta \leq 0.
    \end{equation*}
\end{proof}
Hence by \cref{prop:a_shoen}, $\sqrt{d_{{\rm SW}_1}}$ is Hilbertian. It is important to observe that the proof that $\sqrt{d_{{\rm SW}_1}}$ is Hilbertian does not lead to an explicit feature map and feature space while for $d_{{\rm SW}_2}$ we have an explicit construction. 

\section{Universality with ${\rm SW}_1$} \label{sec:a_C}
\begin{proposition} \label{prop:a_uni_sw1}
    $K_1:X \times X \to \R, (\P,\Q) \mapsto e^{-\gamma {\rm SW}_1(\P,\Q)}$ is universal for $(X,d_{LP})$ where $d_{LP}$ is the Lévy–Prokhorov distance. 
\end{proposition}
Since we do not know explicitely the feature map nor the feature space for $\sqrt{d_{{\rm SW}_1}}$ we need to take a little detour for the separability requirement.
\begin{proof}
  Since $\sqrt{d_{{\rm SW}_1}}$ is Hilbertian, there is a feature map $\Phi: \cM_+^1(\Omega_r) \to \cF$ where $\cF$ is a Hilbert space such that $\sqrt{d_{{\rm SW}_1}}(\P,\Q) =
  \|\Phi(\P) - \Phi(\Q)\|_{\cF}$. Since $d_{{\rm SW}_1}$ satisfies the identity
  of indiscernibles, $\sqrt{d_{{\rm SW}_1}}$ too, which implies that
  $\Phi$ is injective and hence $\sqrt{d_{{\rm SW}_1}}$ is a distance. \cite{nadjahi2019asymptotic} showed that $d_{{\rm SW}_1}$ metrizes the weak convergence and by continuity of the square root $\sqrt{d_{{\rm SW}_1}}$ too. Therefore, we can reason as for $K_2$ except that we need to show that $\cF$ is separable. By Proposition 2.2 \cite{smola1998learning}, $\cF$ can be chosen as the RKHS induced by the kernel
  $$
  \tilde{k}(\P,\Q) = d_{{\rm SW}_1}(\P_{0},\P) + d_{{\rm SW}_1}(\P_{0},\Q)  - d_{{\rm SW}_1}(\P,\Q) = 2
  \langle \Phi(\P) - \Phi(\P_0), \Phi(\Q) - \Phi(\P_0) \rangle_{\cF}.
  $$
  for $\P_0 \in \cM_+^1(\Omega_r)$. By Lemma 4.33
  in \cite{steinwart2008support}, if the input space is separable and
  $\tilde{k}$ is continuous then the RKHS is separable. Here the input space is
  $\cM_+^1(\Omega_r)$ which is weakly compact (since we have assumed $\Omega_r$ compact) hence weakly separable. Furthermore $\Phi$ continuous w.r.t $\sqrt{d_{{\rm SW}_1}}$ implies $\tilde{k}$ continuous (Lemma 4.29 in \cite{steinwart2008support}).
\end{proof}

\section{Practical implementation of the Sliced Wasserstein kernel} \label{sec:a_prac_sw}
Computing SW kernel values requires to evaluate the SW distance. In this section we detail how. The following lemma explain how the generalised inverse of the cumulative distribution function of an empirical measure can be evaluated. 

\begin{lemma}
    We consider the distribution $\hat\alpha  = \sum_{k=1}^{+\infty} p_k
    \delta_{x_k}$ where $(x_k)_{k \geq 1} \in \R^{\N}$ and  $(p_k)_{k \geq 1}
    \in \R^{\N}_+$ such that $\sum_{k \geq 1} p_k = 1$. We refer to the sorted sequence as $(x_{(k)})_{k \geq 1} \in \R^{\N}$
    with reordered weights $(p_{(k)})_{k \geq 1}$ and define $\displaystyle
    s_k = \sum_{i=1}^k p_{(i)}$, $s_0 = 0$. Then, we have
    $$
    F_{\hat\alpha}(x)=\P_{\hat\alpha}(X \leq x) =
    \sum_{k=1}^{+\infty}p_k1_{x_k \leq x} = s_{\max (k \mid x_{(k)} \leq x)} \qquad \left(x \in \R\right),
    $$
    $$
    F_{\hat\alpha}^{[-1]}(t)=\inf \left\{x \in \mathbb{R}:
    \sum_{k=1}^{+\infty}p_k1_{x_k \leq x} \geq t\right\}  =  \{x_{(k)} \mid
    s_{k-1} < t \leq s_{k}, k \geq 1\}, \qquad F_{\hat\alpha}^{[-1]}(0) = -\infty.
    $$
    
    In particular, for a uniform empirical distribution on a set of $n$ points
    $\displaystyle \hat\alpha = \frac{1}{n} \sum_{i=1}^n \delta_{x_i}$, we get
    
    \begin{equation} \label{eq:unif_inv}
    \begin{aligned}
      F_{\hat\alpha}(x) &= \frac{\max (k \mid x_{(k)} \leq x)}{n} \qquad \left(x \in \R\right), \\
      F_{\hat\alpha}^{[-1]}(t) &= x_{(k)} \quad \text{ if } \frac{k-1}{n} < t \leq \frac{k}{n} \quad 1 \leq k \leq n \quad -\infty \text{ if } t=0 \quad \left(t \in [0,1]\right).
    \end{aligned}
    \end{equation}
    $F_{\hat\alpha}^{[-1]}(t)$ is a piece-wise constant function on the uniform
    grid with step $\displaystyle \left(\frac{k-1}{n}, \frac{k}{n}\right], 1 \leq
    k \leq n$.
\end{lemma}
We can use \cref{eq:unif_inv} to compute the W distance and approximate the SW distance between empirical distributions.

\paragraph{Balanced setting.} For two empirical distributions with the same number of points: $\displaystyle \hat\alpha =
\frac{1}{n} \sum_{i=1}^n \delta_{x_i}$, $\displaystyle \hat\beta =
\frac{1}{n} \sum_{i=1}^n \delta_{y_i}$, $x_i, y_i \in \R$ exploiting \cref{eq:unif_inv}, we get
\begin{equation}
  \begin{aligned} \label{eq:sorted_w}
    d_{W_{p}}(\hat\al, \hat\be)^{p}  &= \int_{0}^1 \left|F_{\hat\alpha}^{[-1]}(t) - F_{\hat\beta}^{[-1]}(t)\right|^p dt \\
    &=  \sum_{k=1}^n \int_{\frac{k-1}{n}}^{k/n} \left|F_{\hat\alpha}^{[-1]}(t) - F_{\hat\beta}^{[-1]}(t)\right|^pdt \\ 
    &=  \sum_{k=1}^n \int_{\frac{k-1}{n}}^{k/n} \left|x_{(k)} - y_{(k)}\right|^pdt \\ 
    &= \frac{1}{n}\sum_{k=1}^n \left|x_{(k)} - y_{(k)}\right|^p
  \end{aligned}
\end{equation}
Now we consider $\displaystyle \hat\alpha =
\frac{1}{n} \sum_{i=1}^n \delta_{x_i}$, $\displaystyle \hat\beta =
\frac{1}{n} \sum_{i=1}^n \delta_{y_i}$, with points in $\R^r$ instead of $\R$.  We recall that for
$\theta \in S^{r-1}$, we denote by $\theta^*$ the map $\theta^*: \R^r \to
\R, x \mapsto \langle x,\theta\rangle$ and by $\theta^*_{\#}$ the associated
push-forward operator\footnote{If $X \sim \al$, $\langle X,
  \theta\rangle \sim \theta^*_{\#}\al$.}. We have $\theta^*_{\#}\alpha = \frac{1}{n}\sum_{i=1}^n
\delta_{\langle x_i, \theta \rangle}$. We denote by $(x_{(i)}^{\theta})_{i=1}^n$ the sorted sequence after projection, i.e. $x_{i}^{\theta} := \langle x_{i}, \theta \rangle$ for all $i=1, \ldots, n$ and 
\begin{equation}
  x_{(1)}^{\theta} \leq \ldots \leq x_{(n)}^{\theta}
\end{equation}
By \cref{eq:sorted_w} the Sliced Wasserstein distance is given by
\begin{equation}
    d_{{\rm SW}_{p}}(\hat\al, \hat\be)^{p} = \frac{1}{n} \sum_{k=1}^n \int_{S^{r-1}} \left| x_{(k)}^{\theta} - y_{(k)}^{\theta} \right|^pd\theta
\end{equation}
which can be approximated as 
\begin{equation}
    d_{{\rm SW}_{p}}(\hat\al, \hat\be)^{p} \approx \frac{1}{nM}\sum_{m=1}^M \sum_{k=1}^n \left|x_{(k)}^{\theta_m} - y_{(k)}^{\theta_m} \right|^p
\end{equation}
where $\theta_m \sim^{i.i.d} \cU(S^{d-1})$.

\paragraph{Unalanced setting.} For two empirical distributions: $\displaystyle \hat\alpha =
\frac{1}{n} \sum_{i=1}^n \delta_{x_i}$, $\displaystyle \hat\beta =
\frac{1}{m} \sum_{i=1}^m \delta_{y_i}$, $x_i, y_i \in \R$, as the functions $F_{\hat\alpha}^{[-1]}$ and $F_{\hat\beta}^{[-1]}$ are piece-wise constant functions on the uniform grid with respectively $n$ and $m$ points, $t \mapsto
\left|F_{\hat\alpha}^{[-1]}(t) - F_{\hat\beta}^{[-1]}(t)\right|$ is piece-wise constant on the intervals defined by the intersection of the two uniform grids and the value it takes on each interval is given by the ordered sequences. The computation of $d_{W_{p}}(\hat\al, \hat\be)^{p}$ and the approximation of $d_{{\rm SW}_{p}}(\hat\al, \hat\be)^{p}$ when we go to $\R^r$ is therefore performed in a fashion similar to the balanced setting. The only additional step is to deal with the intersection of the grids. 

\section{Main proofs of \cref{sec:excess_risk}} \label{sec:a_main} 
\subsection{Discussion around the Hölder assumption} \label{subsec:a_holder}
Assumption \cref{asm:holder_d_subst} might make the reader think that its validity depends on the chosen Hilbertian metric $d$ while it does not. Indeed, as noticed in \citet{szabo2015two} we can write 
\begin{equation*}
  \begin{aligned}
    \|K_{\P} - K_\Q\|_{\cH_K}^2 &= K(\P,\P) + K(\Q,\Q) - 2K(\P,\Q) \\
    &= 2\left(q(0) - q(d(\P,\Q))\right).
  \end{aligned}  
\end{equation*}
Therefore, the Hölder condition on $K$ can be replaced by a regularity condition on $q$ only. There are constants $L>0$ and $h \in (0,1]$ such that
\begin{equation} \label{asm:holder_simple}
    \frac{q(0) - q(x)}{x^{2h}} \leq L, \quad \forall x \in \R_+^*
\end{equation}
which is independent of the underlying distance $d$. For example, it holds with $h=\beta$ for $q:x \mapsto e^{-\gamma x^{2\beta}}$ as shown below in \cref{prop:app_beta_kernel}. Inequality \cref{asm:holder_d_subst} is important as the derivations for the excess risk bound involve multiple times the difference $\|K_{\P} - K_{\hat\P_n}\|_{\cH_K}$ which can be controlled by $d(\P,\hat\P_n)$ thanks to the regularity assumption.

\begin{proposition} \label{prop:app_beta_kernel} For $\gamma > 0$ and $\beta \in [0,1]$, $q: \R_+ \to \R, x \mapsto e^{-\gamma x^{2\beta}}$ satisfies
\begin{equation*} 
    \frac{q(0) - q(x)}{x^{2h}} \leq L, \quad \forall x\in \R_+^*
\end{equation*}
for $h=\beta$ and some $L>0$.
\end{proposition}
\begin{proof}
$x \to \frac{q(0) - q(x)}{x^{2h}}$ is continuous and goes to $0$ as $x$ goes to $+\infty$. Therefore it remains to show that the limit exists when $x \to 0^{+}$. By L'Hôpital's rule, 
$$\lim _{x \rightarrow 0^+} \frac{q(0) - q(x)}{x^{2h}}=\lim _{x \rightarrow 0^+} \frac{1 - e^{-\gamma x^{2\beta}}}{x^{2h}}=\lim _{x \rightarrow 0^+}   \frac{\gamma\beta x^{2\beta-1} e^{-\gamma x^{2\beta}}}{hx^{2h-1}}$$
The largest $h$ for which the limit exists is $h = \beta$.  
\end{proof}

\subsection{Proof of the main theorem} \label{subsec:a_th_proof}
Before proving the main theorem we introduce the approximations of the integral operator.  
\paragraph{Integral operator approximations.} Given the first and second stage sampling inputs we introduce the operators 
\begin{align*}
    \Phi_{\cD} \colon \R^T & \rightarrow \cH_K \qquad & \Phi_{\hat{\cD}} \colon \R^T & \rightarrow \cH_K \\
    \alpha &\mapsto \sum_{t=1}^T \alpha_t K_{\P_t} \qquad & \alpha &\mapsto \sum_{t=1}^T \alpha_t K_{\hat\P_{t,n}}
\end{align*}
their adjoint operators are the first and second stage sampling operators
\begin{align*}
    \Phi_{\cD}^* \colon \cH_K & \rightarrow \R^T \qquad &\Phi_{\hat{\cD}}^* \colon \cH_K & \rightarrow \R^T \\
    f &\mapsto f(\P_t)_{t=1}^T \qquad &f &\mapsto f(\hat\P_{t,n})_{t=1}^T
\end{align*}
The integral operator is approximated at first and second stage respectively by $L_{K,\cD} := \frac{1}{T} \Phi_{\cD} \circ \Phi_{\cD}^*$ and \\ $L_{K,\hat{\cD}} := \frac{1}{T}\Phi_{\hat\cD} \circ \Phi_{\hat\cD}^*$.
\begin{align*}
    L_{K,\cD} \colon \cH_K& \rightarrow \cH_K \qquad &\L_{K,\hat{\cD}} \colon \cH_K& \rightarrow \cH_K \\
    f &\mapsto \frac{1}{T}\sum_{t=1}^T f(\P_t) K_{\P_t} \qquad & f &\mapsto \frac{1}{T}\sum_{t=1}^T f(\hat\P_{t,n}) K_{\hat\P_{t,n}}
\end{align*}
The introduction of the integral operator and its approximations from finite samples is convenient as we can express the solution of the regularized empirical risk minimization problems as follows, 
\begin{equation} \label{eq:a_explicit}
    f_{\cD,\la} = (L_{K,\cD} + \lambda I)^{-1}\frac{\Phi_{\cD}}{T} y \qquad f_{\hat{\cD},\la} = (L_{K,\hat{\cD}} + \lambda I)^{-1}\frac{\Phi_{\hat{\cD}}}{T} y
\end{equation}

\paragraph{Plan of the proof.}
Our goal is to control $\cE(f_{\hat{\cD},\la}) - \cE(f_{\rho})$ the expected excess risk. But since $\cE$ can be decomposed as
\begin{equation}
    \cE(f) = \|f - f_{\rho}\|_{L^2_{\rho_X}}^2 + \cE(f_{\rho}),
\end{equation}
we will consider the decomposition 
\begin{equation}
\begin{aligned}
  \left(\cE(f_{\hat{\cD},\la}) - \cE(f_{\rho})\right)^{1/2} &= \|f_{\hat{\cD},\la} - f_{\rho}\|_{L_2} \\ &\leq \underbrace{\|f_{\hat{\cD},\la} - f_{\cD,\la}\|_{L^2}}_{=:B} + \underbrace{\|f_{\cD,\la} - f_{\rho}\|_{L_2}}_{=:A}.
\end{aligned}
\end{equation}
Both terms can be written in terms of the integral operator and its approximations as defined previously,
\begin{equation}
\begin{aligned}
  A &= \|f_{\cD,\la} - f_{\rho}\|_{L^2} = \left\|L_K^{1/2}\left((L_{K,\cD} + \lambda I)^{-1}\frac{\Phi_{\cD}}{T} y - (L_{K,\cD} + \lambda I)^{-1}(L_{K,\cD} + \lambda I) f_{\rho}\right)\right\|_{\cH_{K}} \\ &\leq \left\|L_K^{1/2}(L_{K,\cD} + \lambda I)^{-1}\left(\frac{\Phi_{\cD}}{T} y - L_{K,\cD}f_{\rho}\right)\right\|_{\cH_{K}} + \lambda \left\|L_K^{1/2}(L_{K,\cD} + \lambda I)^{-1}L_K^{1/2} g_{\rho}\right\|_{\cH_{K}}.
\end{aligned}
\end{equation}
Where we used the isometry $L_K^{1/2}$ and that since $f_{\rho} \in \cH_K$, $f_{\rho} = L_K^{1/2}(g_{\rho})$ for some $g_{\rho} \in L^2_{\rho_X}$. Those quantities do not depend on the sampling rate of the chosen distance for the kernel. Bounding the expected value of $A$ boils down to control for $\E\left[\left\|(L_{K} + \lambda I) (L_{K,\cD} + \lambda I)^{-1}\right\| \right]$ and $\E\left[\left\|(L_{K} + \lambda I)^{-1}\left(\frac{\Phi_{\cD}}{T}y - L_{K,\cD}f_{\rho}\right)\right\| \right]$, Proposition 5 \citet{fang2020optimal}. This can be obtained through concentration inequalities. Term A is the excess risk for standard Kernel Ridge Regression (1-stage sampling) and we can directly plug known results, for example Lemma 4 \citet{fang2020optimal}:
\begin{equation}
E\left[\left\|f_{\cD, \lambda}-f_{\rho}\right\|_{L^2}\right] \leq C\left\{ \left(\frac{\mathcal{B}_{T, \lambda}}{\sqrt{\lambda}}+1\right)^{2} \frac{A}{\kappa} \mathcal{B}_{T, \lambda}+ \left\|g_{\rho}\right\|_{L_2}\left(\mathcal{B}_{T, \lambda}+\sqrt{\lambda}\right)\right\},
\end{equation}
where $C$ is a constant that does not depend on any parameter, $\mathcal{B}_{T, \lambda}=\frac{2 \kappa}{\sqrt{T}}\left(\frac{\kappa}{\sqrt{T \lambda}}+\sqrt{\mathcal{N}(\lambda)}\right)$, $\mathcal{N}(\lambda)$ is the effective dimension, $\mathcal{N}(\lambda)=\operatorname{Tr}\left(L_{K}\left(L_{K}+\lambda I\right)^{-1}\right)$ $\forall \lambda>0$ and $A$ is such that $|y| \leq A$, $\rho$ a.s.

Term B is the error induced by the fact that we only access the covariates $(\P_t)_{t=1}^T$ through empirical approximations $(\hat\P_{t,n})_{t=1}^T$, this term is specific to distribution regression. We follow a decomposition of $B$ similar to the one used in \cite{fang2020optimal}, but while they focused on MMD-based kernels we adapt the decomposition to incorporate the sample rate of any distance substitution kernels. 

\begin{equation}
\begin{aligned}
  B &= \|f_{\hat{\cD},\la} - f_{\cD,\la}\|_{L^2_{\rho_X}} = \left\|(L_{K,\hat{\cD}} + \lambda I)^{-1}\frac{\Phi_{\hat{\cD}}}{T} y - (L_{K,\cD} + \lambda I)^{-1}\frac{\Phi_{\cD}}{T} y\right\|_{L^2_{\rho_X}} \\&= \left\|(L_{K,\hat{\cD}} + \lambda I)^{-1}\left(\frac{\Phi_{\hat{\cD}}}{T} y - \frac{\Phi_{\cD}}{T} y  \right) + \left((L_{K,\hat{\cD}} + \lambda I)^{-1} - (L_{K,\cD} + \lambda I)^{-1} \right)\frac{\Phi_{\cD}}{T}y\right\|_{L^2_{\rho_X}} \\ &= \left\|(L_{K,\hat{\cD}} + \lambda I)^{-1}\left(\frac{\Phi_{\hat{\cD}}}{T} y - \frac{\Phi_{\cD}}{T} y  \right) + (L_{K,\hat{\cD}} + \lambda I)^{-1}\left(L_{K,\cD} - L_{K,\hat{\cD}}\right)f_{\cD,\la}\right\|_{L^2} \\&\leq \underbrace{\left\|L_K^{1/2}(L_{K,\hat{\cD}} + \lambda I)^{-1}\left(\frac{\Phi_{\hat{\cD}}}{T} y - \frac{\Phi_{\cD}}{T} y  \right)\right\|_{\cH_K}}_{=:C} + \underbrace{\left\|L_K^{1/2}(L_{K,\hat{\cD}} + \lambda I)^{-1}\left(L_{K,\cD} - L_{K,\hat{\cD}}\right)f_{\cD,\la}\right\|_{\cH_{K}}}_{:=D}.
\end{aligned}
\end{equation}

Where we used that for two invertible operators $A^{-1} - B^{-1} = A^{-1}(B-A)B^{-1}$, \cref{eq:a_explicit} and the isometry $L_K^{1/2}$. We aim at bounding $\E \|f_{\hat{\cD},\la} - f_{\cD,\la}\|_{L^2}$ in expectation, and we will see that it involves the sample complexity $d(\P,\hat\P_n)$. Let us start with term $C$, by the Cauchy-Schwarz inequality, 
\begin{equation}
\begin{aligned}
  \E[C] &= \E\left[\left\|L_K^{1/2}(L_{K,\hat{\cD}} + \lambda I)^{-1}\left(\frac{\Phi_{\hat{\cD}}}{T} y - \frac{\Phi_{\cD}}{T} y  \right)\right\|_{\cH_K}\right] \\ &\leq \E\left[\left\|L_K^{1/2}(L_{K,\hat{\cD}} + \lambda I)^{-1}\right\|_{op} \left\|\left(\frac{\Phi_{\hat{\cD}}}{T} y - \frac{\Phi_{\cD}}{T} y  \right)\right\|_{\cH_K}\right] \\ &\leq \left(\E\left[\left\|L_K^{1/2}(L_{K,\hat{\cD}} + \lambda I)^{-1} \right\|_{op}\right]^2\right)^{1/2}\left(\E\left[\left\|\frac{\Phi_{\hat{\cD}}}{T} y - \frac{\Phi_{\cD}}{T} y \right\|_{\cH_K}\right]^2\right)^{1/2},
\end{aligned}
\end{equation}
where $\|\cdot\|_{op}$ denotes the operator norm. For the second term we have, 
\begin{equation}
\begin{aligned}
  \left(\E\left[\left\|\frac{\Phi_{\hat{\cD}}}{T} y - \frac{\Phi_{\cD}}{T} y \right\|_{\cH_K}^2\right]\right)^{1/2} &= \left(\E\left[\left\|\frac{1}{T}\sum_{t=1}^T y_t\left(K(\P_t,\cdot) - K(\hat\P_{t,n},\cdot)\right) \right\|_{\cH_K}\right]^2\right)^{1/2} \\ &\leq \frac{A}{\sqrt{T}} \left(\sum_{t=1}^T \E\left\|K(\P_t,\cdot) - K(\hat\P_{t,n}) \right\|_{\cH_K}^2\right)^{1/2} \\ &\leq \frac{AL}{\sqrt{T}} \left(\sum_{t=1}^T \E\left[d(\P_t,\hat\P_{t,n})^{2h}\right]\right)^{1/2} \\  &\leq \frac{AL}{\sqrt{T}} \left(\sum_{t=1}^T \E\left[d(\P_t,\hat\P_{t,n})^{2}\right]^h\right)^{1/2} \\  &\leq AL\alpha(n)^{h/2}.
\end{aligned}
\end{equation}
where we used the Jensen inequality ($h < 1$), $|y| \leq A$, \cref{asm:holder_d_subst} and the control over the sample complexity of $d$. For the first term, Lemma 8 in \citet{fang2020optimal} shows
\begin{equation}
\begin{aligned}
  \left\|L_K^{1/2}(L_{K,\hat{\cD}} + \lambda I)^{-1} \right\|^2_{op} \leq \frac{3\Sigma_{\cD}^2}{\lambda^2} + \frac{3}{\lambda^3}\|L_{K,\cD} - L_{K,\hat\cD}\|^2_{op} +  \frac{3}{\lambda},
\end{aligned}
\end{equation}
where $\Sigma_{\cD}=\left\|\left(L_{K}+\lambda I\right)^{-\frac{1}{2}}\left(L_{K}-L_{K, \cD}\right)\right\|_{op}$. We now show that the middle term can be controlled with $\alpha(n)$. For all $f \in \cH_K$,
\begin{equation}
\begin{aligned}
  \|(L_{K,\cD} - L_{K,\hat\cD})(f)\|^2_{\cH_K} &= \left\|\frac{1}{T}\sum_{t=1}^T K(\P_t,\cdot)f(\P_t) - K(\hat\P_{t,n},\cdot)f(\hat\P_{t,n}) \right\|_{\cH_K}^2 \\ &\leq \frac{1}{T} \sum_{t=1}^T \left\|K(\P_t,\cdot)f(\P_t) - K(\hat\P_{t,n})f(\hat\P_{t,n}) \right\|_{\cH_K}^2 \\ &\leq \frac{2}{T} \sum_{t=1}^T (f(\P_t) - f(\hat\P_{t,n}))^2\left\|K(\P_t,.)\right\|_{\cH_K}^2 + f(\hat\P_{t,n})^2\left\|K(\P_t,.) -K(\hat\P_{t,n})\right\|_{\cH_K}^2 \\
  &\leq \frac{2\kappa^2}{T} \sum_{t=1}^T (f(\P_t) - f(\hat\P_{t,n}))^2 + \|f\|_{\cH_K}^2\left\|K(\P_t,.) -K(\hat\P_{t,n})\right\|_{\cH_K}^2 \\
  &\leq \frac{4\kappa^2L^2\|f\|_{\cH_K}^2}{T} \sum_{t=1}^T d(\P_t,\hat\P_{t,n})^{2h},
\end{aligned}
\end{equation}
where we used that $K(\P,\P)$ is bounded by $\kappa^2$ for all $\P \in X$ and
\begin{equation}
\begin{aligned}
    f(\hat\P_{t,n})^2 &\leq \|f\|_{\cH_K}^2\left\|K(\hat\P_{t,n},.) \right\|_{\cH_K}^2 \leq \|f\|_{\cH_K}^2\kappa^2, \\
    \left(f(\P_t) - f(\hat\P_{t,n})\right)^2 &\leq \|f\|_{\cH_K}^2\left\|K(\P_t,.) - K(\hat\P_{t,n},.) \right\|_{\cH_K}^2 \leq L^2\|f\|_{\cH_K}^2d(\P_t,\hat\P_{t,n})^{2h},
\end{aligned}
\end{equation}
by the Hölder assumption. Hence,
\begin{equation}
  \|L_{K,\cD} - L_{K,\hat\cD}\|^2_{op} \leq \frac{4\kappa^2L^2}{T} \sum_{t=1}^T d(\P_t,\hat\P_{t,n})^{2h},
\end{equation}
which implies
\begin{equation}
  \E\left[\|L_{K,\cD} - L_{K,\hat\cD}\|^2_{op}\right] \leq \frac{4\kappa^2 L^2}{T}\sum_{t=1}^T \E[d(\P_t,\hat\P_{t,n})^{2h}] \leq 4\kappa^2 L^2 \alpha(n)^{h},
\end{equation}
and finally
\begin{equation}
\begin{aligned}
  \E\left\|L_K^{1/2}(L_{K,\hat{\cD}} + \lambda I)^{-1} \right\|^2_{op} \leq \frac{3\E\left[\Sigma_{\cD}^2\right]}{\lambda^2} + \frac{12}{\lambda^3}L^2\kappa^2\alpha(n)^{h} +  \frac{3}{\lambda}.
\end{aligned}
\end{equation}
Putting the bounds together, we have 
\begin{equation}
    \begin{aligned}
    \E[C] &\leq AL\alpha(n)^{h/2}\sqrt{\frac{3\E[\Sigma_{\cD}^2]}{\lambda^2} + \frac{12}{\lambda^3}L^2\kappa^2\alpha(n)^{h} +  \frac{3}{\lambda}} \\
    &\leq \sqrt{3}AL\alpha(n)^{h/2}\left(\frac{\E[\Sigma_{\cD}^2]^{1/2}}{\lambda} + \frac{2}{\lambda^{3/2}}L\kappa\alpha(n)^{h/2} +  \frac{1}{\sqrt{\lambda}} \right) \\
    &\leq \frac{\sqrt{3}AL\alpha(n)^{h/2}}{\sqrt{\lambda}}\left(\kappa\sqrt{\frac{\mathcal{N}(\lambda)}{T\lambda}} + \frac{2}{\lambda}L\kappa\alpha(n)^{h/2} + 1 \right),
    \end{aligned}
\end{equation}
where we used Lemma 17 in \citet{lin2017distributed}: $\E[\Sigma^2_{\cD}] \leq \frac{\kappa^2\mathcal{N}(\lambda)}{T}$. Now for D, 

\begin{equation}
    \begin{aligned}
        \E[D] &= \left\|L_K^{1/2}(L_{K,\hat{\cD}} + \lambda I)^{-1}\left(L_{K,\cD} - L_{K,\hat{\cD}}\right)f_{\cD,\la}\right\|_{\cH_{K}} \\
        &\leq \E_{\cD}\left[ \E_{\hat\cD \mid \cD}\left\{\left\|L_K^{1/2}(L_{K,\hat{\cD}} + \lambda I)^{-1}\right\|_{op}\left\|L_{K,\cD} - L_{K,\hat{\cD}}\right\|_{op}\right\}\left\|f_{\cD,\la}\right\|_{\cH_{K}}\right] \\
        &\leq \E_{\cD}\left[ \E_{\hat\cD \mid \cD}\left\{\left\|L_K^{1/2}(L_{K,\hat{\cD}} + \lambda I)^{-1}\right\|^2_{op}\right\}^{1/2}\E_{\hat\cD \mid \cD}\left\{\left\|L_{K,\cD} - L_{K,\hat{\cD}}\right\|^2_{op}\right\}^{1/2}\left\|f_{\cD,\la}\right\|_{\cH_{K}}\right].
    \end{aligned}
\end{equation}
All terms have already been dealt with except $\left\|f_{\cD,\la}\right\|_{\cH_{K}}$, 
\begin{equation}
    \begin{aligned}
        \E[D] &\leq \E_{\cD}\left[ \E_{\hat\cD \mid \cD}\left\{\left\|L_K^{1/2}(L_{K,\hat{\cD}} + \lambda I)^{-1}\right\|^2_{op}\right\}^{1/2}2\kappa L \alpha(n)^{h/2}\left\|f_{\cD,\la}\right\|_{\cH_{K}}\right] \\
        &\leq \E_{\cD}\left[ \sqrt{3}\left\{\frac{\Sigma_{\cD}}{\lambda} + \frac{2}{\lambda^{3/2}}L\kappa\alpha(n)^{h/2} +  \frac{1}{\sqrt{\lambda}}\right\}\left(2\kappa L \alpha(n)^{h/2}\right)\left\|f_{\cD,\la}\right\|_{\cH_{K}}\right] \\ 
        &\leq \frac{2\sqrt{3}\kappa L \alpha(n)^{h/2}}{\sqrt{\lambda}} \left(\E_{\cD}\left[\frac{\Sigma_{\cD}^2}{\lambda}\right]^{1/2}\E_{\cD}\left[\left\|f_{\cD,\la}\right\|_{\cH_{K}}^2\right]^{1/2} + \E_{\cD}\left[\left\|f_{\cD,\la}\right\|_{\cH_{K}}\right]\left(\frac{2}{\lambda}L\kappa\alpha(n)^{h/2} +  1 \right) \right) \\
        &\leq \frac{2\sqrt{3}\kappa L \alpha(n)^{h/2}}{\sqrt{\lambda}}\E_{\cD}\left[\left\|f_{\cD,\la}\right\|_{\cH_{K}}^{2}\right]^{1/2} \left(\kappa\sqrt{\frac{\mathcal{N}(\lambda)}{T\lambda}} + \left(\frac{2}{\lambda}L\kappa\alpha(n)^{h/2} +  1 \right) \right).
    \end{aligned}
  \end{equation}
  We conclude with Proposition 7 \citet{fang2020optimal} that states

  \begin{equation}
    \E_{\cD}\left[\left\|f_{\cD,\la}\right\|_{\cH_{K}}^{2}\right]^{1/2} \leq \left(\frac{\mathcal{B}_{T, \lambda}}{\sqrt{\lambda}}+1\right) \frac{25 A}{\kappa} \frac{\mathcal{B}_{T, \lambda}}{\sqrt{\lambda}}+6\left\|g_{\rho}\right\|_{\rho},
  \end{equation}
  hence,
\begin{equation}
  \E[D] \leq \sqrt{3}\left\{\left(\frac{\mathcal{B}_{T, \lambda}}{\sqrt{\lambda}}+1\right) \frac{25 A}{\kappa} \frac{\mathcal{B}_{T, \lambda}}{\sqrt{\lambda}}+6\left\|g_{\rho}\right\|_{\rho}\right\}\left(\frac{\kappa \sqrt{\mathcal{N}(\lambda)}}{\sqrt{T \lambda}}+\frac{2\kappa L}{\lambda}\alpha(n)^{h/2}+1\right)\frac{2\kappa L}{\sqrt{\lambda}}\alpha(n)^{h/2}.
\end{equation}
Putting C and D together, we get
\begin{equation}
\begin{aligned}
  \E[B] &\leq \E[C] + \E[D]  \\
  &\leq  \frac{\sqrt{3}L\alpha(n)^{h/2}}{\sqrt{\lambda}}\left(\kappa\sqrt{\frac{\mathcal{N}(\lambda)}{T\lambda}} + \frac{2}{\lambda}L\kappa\alpha(n)^{h/2} +  1 \right)\left[A + 2\kappa\left\{\left(\frac{\mathcal{B}_{T, \lambda}}{\sqrt{\lambda}}+1\right) \frac{25 A}{\kappa} \frac{\mathcal{B}_{T, \lambda}}{\sqrt{\lambda}}+6\left\|g_{\rho}\right\|_{\rho}\right\} \right] \\
  &\leq  \frac{\sqrt{3}L\alpha(n)^{h/2}}{\sqrt{\lambda}}\left(\kappa\sqrt{\frac{\mathcal{N}(\lambda)}{T\lambda}} + \frac{2}{\lambda}L\kappa\alpha(n)^{h/2} +  1 \right)\left[A + 12\kappa\left\|g_{\rho}\right\|_{\rho} + 50 A \left(\frac{\mathcal{B}_{T, \lambda}}{\sqrt{\lambda}}+1\right)\frac{\mathcal{B}_{T, \lambda}}{\sqrt{\lambda}} \right].
\end{aligned}
\end{equation}

\subsection{Sample complexity MMD}  \label{subsec:a_sample_mmd}
\begin{proposition} 
If $K_{\Omega}$ is bounded, $\sup_{x \in \Omega}K_{\Omega}(x, x) \leq B$ then $d_{\rm MMD}(\P, \Q) = \|\mu_{K_{\Omega}}(\P) - \mu_{K_{\Omega}}(\Q))\|_{\cH_{\Omega}}$ is such that
\begin{equation*}
\E[d_{\rm MMD}(\P, \hat\P_n)^2] \leq \frac{B}{n},
\end{equation*}
where $\P \in X$, $\hat\P_n = \frac{1}{n} \sum_{i=1}^n \delta_{x_i}$, $x_i \sim^{i.i.d} \P$.
\end{proposition}
\begin{proof}
We denote by $\psi$ the canonical feature map of the kernel $K_{\Omega}$: $K_{\Omega}(x,y) = \langle \psi(x), \psi(y) \rangle_{\cH_{\Omega}}$. Then, 
\begin{equation*}
    \begin{aligned}
        \E[d_{\rm MMD}(\P, \hat\P_n)^2] &= \E\left[\left\|\mu_{K_{\Omega}}(\P) -  \mu_{K_{\Omega}}(\hat\P_n)\right\|_{\cH_{\Omega}}^2\right] \\ 
        &= \E\left[\left\|\frac{1}{n} \sum_{i=1}^n \left(\psi(X_i) -  \mu_{K_{\Omega}}(\P)\right) \right\|_{\cH_{\Omega}}^2\right] \\
        &= \frac{1}{n^2}\E\left[\sum_{i=1}^n \left\|\psi(X_i) -  \mu_{K_{\Omega}}(\P) \right\|_{\cH_{\Omega}}^2 + 2\sum_{1 \leq i < j \leq n} \langle \psi(X_i) -  \mu_{K_{\Omega}}(\P), \psi(X_j) -  \mu_{K_{\Omega}}(\P) \rangle_{\cH_{\Omega}}\right] \\
        &= \frac{1}{n}\E_{X \sim \P}\left[\left\|\psi(X) -  \mu_{K_{\Omega}}(\P) \right\|_{\cH_{\Omega}}^2 \right] \\ 
        &= \frac{1}{n}\left\{\E_{X \sim \P}\left[\left\|\psi(X)\right\|_{\cH_{\Omega}}^2\right] + \left\|\mu_{K_{\Omega}}(\P) \right\|_{\cH_{\Omega}}^2 - 2 \E_{X \sim \P}\left[\langle \psi(X), \mu_{K_{\Omega}}(\P) \rangle_{\cH_{\Omega}} \right]\right\} \\ 
        &= \frac{1}{n}\left\{\E_{X \sim \P}[K_{\Omega}(X,X)] - \left\|\mu_{K_{\Omega}}(\P) \right\|_{\cH_{\Omega}}^2  \right\} \\
        &\leq \frac{1}{n}\E_{X \sim \P}[K_{\Omega}(X,X)] \\
        &\leq \frac{B}{n}.
    \end{aligned}
\end{equation*}
Where we use the fact that the points are i.i.d, that
\begin{equation*}
        \E_{X \sim \P}\left[\langle \psi(X), \mu_{K_{\Omega}}(\P) \rangle_{\cH_{\Omega}} \right] = \E_{X \sim \P}\left[\mu_{K_{\Omega}}(\P)(X)\right] = \left\|\mu_{K_{\Omega}}(\P) \right\|_{\cH_{\Omega}}^2,
\end{equation*}
and that for $X$ independent of $Y$ following $\P$, 
\begin{equation*}
        \E_{X,Y}\langle \psi(X) -  \mu_{K_{\Omega}}(\P), \psi(Y) -  \mu_{K_{\Omega}}(\P) \rangle_{\cH_{\Omega}} = \langle \E_{X}\left[\psi(X)\right] -  \mu_{K_{\Omega}}(\P), \E_{Y}\left[\psi(Y)\right] -  \mu_{K_{\Omega}}(\P) \rangle_{\cH_{\Omega}} = 0.
\end{equation*}
\end{proof}
It is possible to obtain a bound in probability instead of expectation using McDiarmid's inequality Theorem 8 \cite{gretton2012kernel}. In \citet{fang2020optimal} they convert the bound in probability to a bound in expectation using $\E[X] = \int_{0}^{+\infty} \P(X > t)dt$ which leads to the less optimal bound
\begin{equation}
    \E\left[d_{\rm MMD}(\P, \hat\P_n)^{2}\right] \leq 2(2 + \sqrt{\pi})\frac{B}{n}
\end{equation}

\subsection{Proof of \cref{cor:simplified_bound}} \label{subsec:a_simplified_bound}
In this section $\vee$ denotes the maximum and $\wedge$ the minimum between two elements.

\begin{corollary} 
    Under the assumptions of \cref{th:main}, for $\alpha(n) = c n^{-\beta}$, $\beta, c > 0$, taking $\lambda = \frac{1}{\sqrt{T}} \vee \frac{1}{n^{h\beta/2}}$, the bound can be simplified as 
    \begin{equation*}
        \E\|f_{\hat{\cD},\lambda} - f_{\rho}\|_{L^2_{\rho_X}} 
        \leq C\left(\frac{1}{\sqrt[4]{T}} + \frac{1}{n^{h\beta/4}} \right)\left(\left\|f_{\rho}\right\|_{\cH_K} + 1\right)
    \end{equation*}
    where $C$ is a constant that depends only on $M$, $\kappa$ and $L$.
\end{corollary}
\begin{proof}
    The bound $\mathcal{N}(\lambda)\leq \kappa^2/\lambda$ allows to bound $\frac{\mathcal{B}_{T, \lambda}}{\sqrt{\lambda}} \leq \frac{2\kappa^2}{\lambda}\left(\frac{1}{T} + \frac{1}{\sqrt{T}}\right) \leq \frac{4\kappa^2}{\sqrt{T}\lambda}$. Adding $\alpha(n) = cn^{-\beta}$ into the bound leads to the result. 
    \begin{equation*}
        \begin{aligned}
        \E\|f_{\hat{\cD},\lambda} - f_{\rho}\|_{L^2_{\rho_X}} &\leq  C\bigg\{\frac{1}{n^{h\beta/2}\lambda^{1/2}}\left(\frac{1}{\lambda}\left(\frac{1}{T^{1/2}} + \frac{1}{n^{h\beta/2}}\right) +  1 \right)\left(\left\|f_{\rho}\right\|_{\cH_K} + \left(\frac{1}{\sqrt{T}\lambda}+1\right)\frac{1}{\sqrt{T}\lambda} \right) \\ &+\left(\frac{1}{\sqrt{T}\lambda}+1\right)^{2} \frac{1}{\sqrt{T\lambda}} + \sqrt{\lambda}\left(\frac{1}{\sqrt{T}\lambda}+1\right)\left\|f_{\rho}\right\|_{\cH_K}\bigg\}
        \end{aligned}
    \end{equation*}
    Note that $\lambda^{-1} = \sqrt{T} \wedge n^{h\beta/2} \leq \sqrt{T}$ thus $\frac{1}{\sqrt{T}\lambda} \leq 1$, $\frac{1}{\sqrt{T\lambda}} \leq T^{-1/4}$, furthermore $\sqrt{\lambda} = T^{-1/4} \vee n^{-h\beta/4} \leq T^{-1/4} + n^{-h\beta/4}$ hence
    \begin{equation*}
        \E\|f_{\hat{\cD},\lambda} - f_{\rho}\|_{L^2_{\rho_X}} \leq  C\bigg\{\frac{1}{n^{h\beta/2}\lambda^{1/2}}\left(\frac{1}{\lambda n^{h\beta/2}} +  1 \right)\left(\left\|f_{\rho}\right\|_{\cH_K} + 1 \right) + \frac{1}{\sqrt[4]{T}} + \left(\frac{1}{\sqrt[4]{T}} + \frac{1}{n^{h\beta/4}} \right)\left\|f_{\rho}\right\|_{\cH_K}\bigg\}
    \end{equation*}
    We also have  $\lambda^{-1} = \sqrt{T} \wedge n^{h\beta/2} \leq n^{h\beta/2}$ thus $\frac{1}{\lambda n^{h\beta/2}} \leq 1$, $\frac{1}{\sqrt{\lambda} n^{h\beta/2}} \leq \frac{1}{n^{h\beta/4}}$, hence
    \begin{equation*}
        \begin{aligned}
        \E\|f_{\hat{\cD},\lambda} - f_{\rho}\|_{L^2_{\rho_X}} &\leq C\bigg\{\frac{1}{n^{h\beta/4}}\left(\left\|f_{\rho}\right\|_{\cH_K} + 1 \right) +\frac{1}{\sqrt[4]{T}} + \left(\frac{1}{\sqrt[4]{T}} + \frac{1}{n^{h\beta/4}} \right)\left\|f_{\rho}\right\|_{\cH_K}\bigg\} \\ 
        &\leq C\left(\frac{1}{\sqrt[4]{T}} + \frac{1}{n^{h\beta/4}} \right)\left(\left\|f_{\rho}\right\|_{\cH_K} + 1\right)
        \end{aligned}
    \end{equation*}
    which concludes the proof. The notation $C$ absorbs all the terms independent of $n$, $T$ and $\lambda$ throughout the proof. 
\end{proof}

\section{Experimental details} \label{sec:exp_details}
The hyperparameters tested on both experiments are as follows.
\begin{itemize}
\item Regularization for the Kernel Ridge regression ($\lambda$): 25 points between $10^{-8}$ and $100$ in logarithmic scale;
\item $\gamma$ for the Euclidean Gaussian  kernel: 14 points between $10^{-3}$ and $1$ in logarithmic scale;
\item $\gamma$ for the \emph{inner} Gaussian MMD kernel: 14 points between $10^{-6}$ and $100$ in logarithmic scale;
\item $\gamma$ for the \emph{outer} Gaussian MMD kernel: 7 points between $10^{-3}$ and $100$ in logarithmic scale;
\item $\gamma$ for the Gaussian Sliced Wasserstein kernels: 14 points between $10^{-5}$ and $100$ in logarithmic scale;
\end{itemize}
%%%%%%%%%%%%%%%%%%%%%%%%%%%%%%%%%%%%%%%%%%%%%%%%%%%%%%%%%%%%%%%%%%%%%%%%%%%%%%%
%%%%%%%%%%%%%%%%%%%%%%%%%%%%%%%%%%%%%%%%%%%%%%%%%%%%%%%%%%%%%%%%%%%%%%%%%%%%%%%

\clearpage
\section{Additional experiments: Hellinger and total variation distances} \label{sec:additional_exp}

For two probability measures $\P$ and $\Q$ both absolutely continuous with respect to a third probability measure $\mu$, the Hellinger distance between $\P$ and $\Q$ is $$d_H\left(\P, \Q\right)^2 = \frac{1}{2}\int \left(\sqrt{\frac{d\P}{d\mu}} - \sqrt{\frac{d\Q}{d\mu}}\right)^2d\mu.$$
It can be shown that this definition does not depend on $\mu$, so the Hellinger distance between $\P$ and $\Q$ does not change if $\mu$ is replaced with a different probability measure with respect to which both $\P$ and $\Q$ are absolutely continuous. The Hellinger distance is bounded by 1 and it is not difficult to show that for empirical distributions with disjoint supports it is always equal to 1. This explain why we cannot use the Hellinger distance on the Synthetic Mode classification experiment without a density estimation step as the raw empirical distributions built from the mixtures will have disjoint supports with probability 1. For two empirical distributions with the same support (as it is the case in the MNIST classification experiment), the Hellinger distance between $\displaystyle \hat\P = \sum_{i=1}^n\alpha_i \delta_{x_i}$, and $\displaystyle \hat\Q = \sum_{i=1}^n \beta_i \delta_{x_i}$, $x_i \in \R^r$ is $$d_H\left(\hat\P, \hat\Q\right)^2 = \frac{1}{2}\sum_{i=1}^n \left(\sqrt{\alpha_i} - \sqrt{\beta_i}\right)^2 = \frac{1}{2}\left\|\sqrt{\alpha} - \sqrt{\beta}\right\|^2_2,$$ where the square root is applied component-wise. Similarly, the total variation distance is given by $$d_{TV}\left(\hat\P, \hat\Q\right) = \frac{1}{2}\left\|\alpha - \beta\right\|_1.$$ Both $d_H$ and $\sqrt{d_{TV}}$ are Hilbertian, we can therefore plug them in a Gaussian kernel to perform KRR. \cref{table-mnist-appendix} extends the results of \cref{table-mnist} to those two distances. The setting of the experiment is the same as in the main text. Both $d_H$ and $\sqrt{d_{TV}}$ do not incorporate information about differences in the support of the distributions. Therefore, they can perform well when the images are well aligned (no rotation or translation) but their performance degrades significantly when the images are perturbed. This is similar to the performance of the RBF kernel. 

\begin{table}[h]
\caption{Mean accuracy and standard deviation on rotated MNIST and Fashion MNIST with maximum rotation $\theta_\vee = 0, \pi/12, \pi/6$ respectively. Comparison between KRR estimators with standard Gaussian kernel (RBF), doubly Gaussian MMD kernel (MMD) and Gaussian kernels with the distances ${\rm SW}_2$, ${\rm SW}_1$, Hellinger and total variation. 
}
\label{table-mnist-appendix}
\vskip 0.15in
\begin{center}
\begin{sc}
\begin{tabular}{lcccccc}
  \toprule
   MNIST & RBF & MMD  & ${\rm SW}_2$ & ${\rm SW}_1$ & $TV$ & Hellinger  \\
  \midrule
  Raw & $0.9 ~(0.01) $ & $0.79 ~(0.03)$ & $\mathbf{0.93 ~(0.01)}$ & $0.91 ~(0.03)$ & $0.87 ~(0.01)$ & $0.92 ~(0.01)$ \\
  $\theta_{\vee}=\frac{\pi}{12}$ & $0.51 ~(0.03)$ & $0.40 ~(0.05)$ & $\mathbf{0.85 ~(0.02)}$ & $0.66 ~(0.02)$ & $0.44 ~(0.02)$ & $0.58 ~(0.02)$   \\
  $\theta_{\vee}=\frac{\pi}{6}$ & $0.47 ~(0.03)$ & $0.34 ~(0.04)$ & $\mathbf{0.82 ~(0.01)}$ & $0.61 ~(0.01)$ & $0.4 ~(0.03)$ & $0.55 ~(0.02)$ \\
  \bottomrule
\end{tabular}
\end{sc}
\end{center}
\begin{center}
\begin{sc}
\begin{tabular}{lcccccc}
  \toprule
  Fashion & RBF & MMD & ${\rm SW}_2$ & ${\rm SW}_1$ & $TV$ & Hellinger \\
  \midrule
  Raw & $0.82~( 0.01)$ & $0.75~(0.01)$ & $0.81~(0.01)$ & $0.77~(0.01)$ & $0.83~( 0.01)$ & $\mathbf{0.84~( 0.02)}$ \\
  $\theta_{\vee}=\frac{\pi}{12}$ & $0.66~(0.01)$ & $0.48~(0.03)$ & $\mathbf{0.74~(0.01)}$ & $0.60~(0.02)$ & $0.67~(0.02)$ & $0.67~(0.02)$ \\
  $\theta_{\vee}=\frac{\pi}{6}$ & $0.64~(0.01)$ & $0.45~ (0.01$) & $\mathbf{0.70~(0.01)}$ & $0.58~(0.01$) & $0.65~(0.02$) & $0.65~(0.02$) \\
  \bottomrule
\end{tabular}
\end{sc}
\end{center}
\vskip -0.1in
\end{table}
\end{document}